\documentclass{article}
\usepackage{mathptmx}
\usepackage{amsmath}
 \usepackage{amsthm}
 \usepackage{amsfonts}
  \usepackage{bbm}
  \usepackage{url}
\usepackage{subfigure}
\usepackage{graphicx}
\usepackage{color}
\usepackage[left=2.5cm,right=2.5cm,top=2.5cm,bottom=2.5cm]{geometry}
\DeclareMathOperator*{\argmin}{arg\,min}
\DeclareMathOperator{\Tr}{Tr}

\newcommand{\RR}{{\mathbb{R}}}
\newcommand{\EE}{\ensuremath{\mathbb{E}}}

\newcommand{\FF}{\ensuremath{\mathcal F}}
\newcommand{\Var}{\mathbb{V}\mathrm{ar}}

\def\diag{\mathop{\rm diag}\nolimits}%

\newcommand{\CB}[1]{\color{black} #1} 
\newcommand{\CR}[1]{\color{black} #1} 
\newcommand{\CM}[1]{\color{black} #1} 
\newcommand{\rev}[1]{{\color{black}#1}}

\usepackage[mathlines, switch]{lineno}
\newtheorem{thm}{Theorem}[section]
\newtheorem{prop}{Proposition}

\newtheorem{rem}{Remark}%

\newtheorem{definition}{Definition}

\numberwithin{equation}{section}
\newtheorem{hyp}{Assumption}[section]
\newtheorem{lem}{Lemma}[section]
\newtheorem{cor}{Corollary}[section]

\title{High-dimensional ridge regression with random features for non-identically distributed data with a variance profile}

\author{Issa-Mbenard Dabo\footnote{New York University Abu Dhabi, UAE (issa.mbenard.dabo@nyu.edu)}  \  \& J\'er\'emie Bigot \footnote{Institut de math\'ematiques de Bordeaux, France (jeremie.bigot@math.u-bordeaux.fr)}}
\begin{document}

\maketitle

\abstract{\rev{Random feature ridge regression is often analyzed in the high-dimensional regime under the homogeneous sampling model $x_i=\Sigma^{1/2}x_i'$, where the vectors $x_i'$ have iid entries and the same covariance matrix $\Sigma$ is shared by all samples. In this paper, we move beyond this setting and study non-identically distributed data through a variance-profile model in which the training and test covariates have row-dependent diagonal covariance matrices $\Sigma_i=\diag(\gamma_{i1}^2,\ldots,\gamma_{ip}^2)$ and $\widetilde{\Sigma}_i=\diag(\tilde\gamma_{i1}^2,\ldots,\tilde\gamma_{ip}^2)$. Our main contribution is the derivation of asymptotic equivalents for the training and test risks of ridge regression with random features when $n$, $p$, and $m$ grow proportionally. The first set of equivalents is obtained by combining the linear-plus-chaos approximation with traffic-probability arguments, whereas the second set is deterministic and follows from operator-valued free probability through an amalgamation-over-the-diagonal argument. These equivalents are sharp in numerical experiments. They also reveal how heterogeneous variance profiles, including mixture-type profiles inspired by MNIST, can modify generalization and exhibit double-descent behavior when the ridge parameter is small.}}



\section{Introduction}\label{sec:intro}
This study focuses on the  classical statistical learning problem {\CB in a regression framework}. {\CB In this setup, we have a dataset consisting of independent pairs of training data} $(y_i, x_i)_{1 \leq i \leq n}$, where each input $x_i \in \mathbb{R}^p$ is a vector of random variables representing features (or predictors), and each $y_i \in \mathbb{R}$ is a scalar response variable that we aim to predict. The primary goal in such contexts is to construct {\CB an estimator} $\hat{f} : \mathbb{R}^d \to \mathbb{R}$ that leverages the training data to make accurate  {\CB responses for new inputs, where $\hat{f}$ belongs to a given functional class  $\mathcal{F}$ .
In this study, our focus is on fitting the training data using the Random Features (RF) model \cite{rahimi2007random}. 
This corresponds to the choice of the following functional class of estimators:}
$$
\mathcal{F}_{\text{RF}}^{W} = \left\{ f_{\theta}(x) = \sum_{j=1}^{m} \theta_j h\left( w_j^\top x / \sqrt{p} \right) : (\theta_1, \ldots, \theta_m) \in \mathbb{R}^m, x \in \mathbb{R}^p \right\},
$$
where $\theta \in \mathbb{R}^m$ are the trainable parameters (or weights) of the model, and $h : \mathbb{R} \to \mathbb{R}$ is {\CB an activation function} such as ReLU, sigmoid, or tanh, which introduces non-linearity into the model. The vectors $w_1, \ldots, w_m$ represent the {\CB so-called} random features, and they are organized into an $m \times p$ matrix $W = (w_1 | w_2 | \cdots | w_m)^\top $ where the entries of the matrix $W$ are {\CM centered and independent random variables that are identically distributed}. These random features are fixed during the training process,  and they serve as a way of projecting the input data into a new space where linear methods can be applied more effectively. Importantly, the entries of $W$ are uncorrelated with the training data $(y_i, x_i)_{1 \leq i \leq n}$, meaning that the random feature transformation is independent of the specific realization of the data.

The RF model can be seen as a simplified form of a two-layer neural network, where the weights in the first layer (that is the random features $W$) are kept fixed, and only the second layer's weights $\theta$ are trained.

{\CB This model has  gained considerable attention in recent years, and it serves as a theoretical basis for studying  the generalization error in deep learning models. In particular, the RF model} has become a popular tool for understanding the double descent phenomenon in machine learning \cite{Belkin19}, as discussed in several key studies \cite{adlam2020neural, adlam2022random, NEURIPS2020_a03fa308, 9931146, NIPS2017_6857, triple-descent, MeiMonta22, hastie2022surprises, pmlr-v119-d-ascoli20a,meng2024multiple}.

 Existing theoretical studies generally explore the double descent phenomenon within the traditional framework of random matrix theory (RMT). This framework assumes that the number $n$ of training samples , the number  $p$ of predictors, and the number $m$ of random features  all grow proportionally. This approach has been particularly useful in examining the behavior of {\CB various} models in high-dimensional statistical learning \cite{couillet2022random}, where the complexity of the model increases with the dimensions of the data. However, a {\CB common assumption in this  research area} is that the training data $(y_i, x_i)_{1 \leq i \leq n}$ consists of independent and identically distributed (iid) observations. Specifically, it is often assumed that the random predictors $x_i$  {\CB follow the sampling model}
\begin{equation}\label{eq:data_iid}
x_i = \Sigma^{1/2} x_i',
\end{equation}
where $\Sigma$ is a positive definite covariance matrix, and $x_1', \ldots, x_n'$ are random vectors with  iid  entries that have mean zero and unit variance.
\rev{In contrast, the results proved in the present paper do not cover an arbitrary non-diagonal covariance structure. They are derived for a variance-profile model in which the covariance of each sample is diagonal and may vary from one observation to another. More precisely, our asymptotic analysis is carried out for covariance matrices of the form $\Sigma_i=\diag(\gamma_{i1}^2,\ldots,\gamma_{ip}^2)$.}

{\CB Several works, including \cite{peche19, benigni2021eigenvalue,NIPS2017_6857,adlam2020neural}, have studied the performances of the RF model under these two assumptions  through the prism of RMT}. 
 The results established in \cite{peche19, benigni2021eigenvalue} are focused on the behavior of the spectrum of the random matrix   
 {\CB 
$$
H =  h\left( \frac{WX_n^\top}{\sqrt{p}} \right) \in \mathbb{R}^{\textcolor{black}{m \times n}},
$$
as the dimensions $n,m$ and $p$ tend to infinity comparably, where $X_n \in \mathbb{R}^{n \times p}$ denote  the  matrix of predictors, and the function $h$ is applied element-wise to the matrix $\frac{WX_n^\top}{\sqrt{p}}$}. These works prove the convergence of  {\CB the empirical distribution of singular values} of $H$ {\CB towards a limit that is characterized} by a fixed-point equation for the related Stieltjes transform. Nevertheless, the non-linearity induced by the activation function $h$ makes $H$ difficult to study with {\CB conventional methods in RMT}. {\CB To circumvent this issue}, these works provide a simpler matrix,  that we refer to as the {\CB ``linear-plus-chaos" approximation} of $H$, that shares the same limiting singular-value distribution as that of $H$. This powerful result appeared to be very useful in a statistical context as it was used in various works, see e.g.\ \cite{adlam2020neural,MeiMonta22,triple-descent,9931146}, to {\CB obtain} asymptotic equivalents of the training and prediction risks {\CB in the regression framework for the RF model}.

{\CB In this paper, we} aim at providing asymptotic equivalents of the training and test risks in a broader context involving variance profiles (a notion to be defined later) {\CB to model  settings where {\CM the rows of the random matrix $X_n$ are not necessarily iid}. In order to achieve this goal, our {\CB approach} strongly relies on an extension of the ``linear-plus-chaos" {\CB approximation of the matrix $H$ recently proposed} in \cite{DaboMale} when the {\CB data matrix $X_n$ and the random features matrix} $W$ are endowed with variance profiles. 
\rev{This contribution is complementary to two nearby strands of literature. On the one hand, relative to our previous work on linear ridge regression with a variance profile \cite{BigotDaboMale}, the present paper addresses the genuinely nonlinear random feature matrix and therefore requires the linear-plus-chaos approximation together with traffic-theory and operator-valued free-probability tools. On the other hand, some recent deterministic-equivalent analyses of ridge regression and random feature regression allow generic non-diagonal covariance structures \cite{defilippis2024dimension,cheng2024dimension}. Our focus here is different: we treat row-wise diagonal covariance matrices that vary across samples. This setting is especially natural for mixture models, class-dependent heteroscedasticity, and other independent-but-non-identically distributed designs.}

\subsection{Main contributions and related works}

 \subsubsection{The use of random matrices with a variance profile}
 
 Building on our previous work \cite{BigotDaboMale} {\CB on high-dimensional linear regression,  we now} move beyond the traditional iid assumption by introducing a more flexible framework where the random predictors $x_i$ are independent but not identically distributed.  To this end, we {\CB shall assume} that
$$
x_i = \Sigma_i^{1/2} x_i',
$$
where $\Sigma_i$ is a diagonal matrix defined as
$$
\Sigma_i = \underset{1 \leq j \leq p}{\text{diag}}(\gamma_{ij}^2).
$$
Here, the terms $\gamma_{i,1}^2, \ldots, \gamma_{i,p}^2$ represent the variances of the individual components of the random predictors $x_i$, allowing for non-identical distributions across different {\CB feature} vectors. Under this generalized framework, the  matrix $X_n \in \mathbb{R}^{n \times p}$  of predictors can be expressed as
$$
 X_n = \Upsilon_x \circ X_n'  ,
$$
where $\circ$ denotes the Hadamard product (element-wise multiplication) of two matrices. In this formulation, $X_n' = \begin{pmatrix} x_1'^\top \\ \vdots \\ x_n'^\top \end{pmatrix}$ is a random matrix with iid  entries, centered with variance one, representing the underlying randomness of the predictors, while $\Upsilon_x = (\gamma_{ij})$ is a deterministic matrix that modulates the variances of the individual entries in $X_n$.  The so-called  variance profile matrix
$$
\Gamma_{n} = (\gamma_{ij}^2) \in \mathbb{R}^{n \times p},
$$
 captures the variance of each element in $X_n$, and it allows for heterogeneous distributions across different samples. This generalization is  useful for modeling {\CB elements in a data set for which} the assumption of identical distributions {\CB may} not hold. 

\subsubsection{Ridge regression with random features}
 Our analysis {\CB begins by using} the functional class $\FF_{\mathrm{RF}}^{W}$ to fit the training data using ridge regression. {\CB This amounts to solve} the following optimization problem:

$$
\hat{\theta}_{\lambda} = \argmin_{\theta \in \mathbb{R}^m} \left( \frac{1}{n} \left\| Y_n - h\left( \frac{WX_n^\top}{\sqrt{p}} \right)^\top \theta \right \|^2_{ 2} + \lambda \| \theta \|^2 \right),
$$
{\CB where $\lambda > 0$ is a regularization parameter,  and $Y_n = (y_1, \cdots, y_n)^\top  \in \mathbb{R}^n$}. To assess the performances of RF ridge regression in a context where the data may not be identically distributed, we make the assumption that the response variable follows the  linear model:
\begin{equation}
y_i = x_i^\top \beta_\ast + \epsilon_i, \quad 1 \leq i \leq n. \label{eq:linmod}
\end{equation}
Here, {\CB $\varepsilon_n  = (\epsilon_1,\ldots,\epsilon_n)^\top  \in \mathbb{R}^n $} represents a noise vector that is independent of the {\CB matrix of predictors} $X_n$,  with $\mathbb{E}[\varepsilon_n] = 0$ and $\mathbb{E}[\varepsilon_n \varepsilon_n^T] = \sigma^2 I_n$. The  vector  of regression coefficients $\beta_{\ast}$ is also random, and it is assumed to be independent of both $X_n$ and $\varepsilon_n$, characterized by $\mathbb{E}[\beta_{\ast}] = 0$ and
$
\mathbb{E}[\beta_{\ast} \beta_{\ast}^T] = \frac{\alpha^2}{p} I_p,
$
{\CB where $\alpha > 0$ represents the average amount of signal strength.}  By assuming randomness of the  regression coefficients $\beta_{\ast}$, we thus focus on the random-effect hypothesis, which corresponds to an average case analysis over a set of dense regression coefficients  as argued in \cite{17-AOS1549}. {\CB Note the generalization of our approach to data sampled from the regression model $y_i = f_{\ast}(x_i) + \epsilon_i$, with $f_{\ast}$ a non-linear function,  is beyond the scope of this paper. We leave this perspective for future works, and for hints in this direction we refer to \cite{MeiMonta22} where the estimation of non-linear models using RF ridge regression has been studied the iid setting.}

The primary aim of this paper is then to evaluate the performance of RF ridge regression within the framework of the linear model stated above,  by specifically analyzing the training risk $E_{train}(\lambda)$ and the testing (or predictive) risk $E_{test}(\lambda)$. These risks are defined as follows:
\begin{eqnarray} 
E_{train} (\lambda) &=& \frac{1}{n}\mathbb{E} \left[\left\| Y_n - h\left( \frac{WX_n^\top}{\sqrt{p}} \right)^\top \hat{\theta}_{\lambda} \right\|_2^2\right],  \label{eq:train} \\
E_{test}(\lambda) &=&  \frac{1}{\tilde{n}} \sum_{i=1}^{\tilde{n}} \mathbb{E} \left[\left( \tilde y_i - h\left( \frac{W\tilde{x}_i}{\sqrt{p}} \right)^\top \hat{\theta}_{\lambda}\right)^2\right], \label{eq:test}
\end{eqnarray}
where the pairs  $(\tilde{x_1},\tilde{y}_1),...,(\tilde{x}_{\tilde{n}},\tilde{y}_{\tilde{n}})$  represent new samples that are independent of the training data $(Y_n, X_n)$, and that also follows the linear model $\tilde{y}_i = \tilde{x}_i^\top \beta_\ast + \tilde{\varepsilon}_i$ for $i=1,\ldots,\tilde{n}$. In this context, $\tilde{x}_i$ is also defined as $\widetilde{\Sigma}^{1/2}_i \tilde{x}'_i$, where $\tilde{x}'_i \in \mathbb{R}^p$ is a random vector with iid  centered components, each having unit variance. The covariance structure of $\tilde{x}_i$ is represented by $\widetilde{\Sigma}_i = \text{diag}(\tilde{\gamma}_{i,1}^2, \ldots, \tilde{\gamma}_{i,p}^2)$.

The central contribution of this paper lies in deriving two types of asymptotic equivalents for $E_{train}(\lambda)$ and $E_{test}(\lambda)$. The first (respectively the second) type of asymptotic equivalents  referred to as ``lozenge equivalents'' (respectively ``square equivalents'') consists in two values,  denoted as $E^\lozenge_{train}(\lambda)$ and $E^\lozenge_{test}(\lambda)$ (respectively $E^\square_{train}(\lambda)$ and $E^\square_{test}(\lambda)$), that approximate the training and predictive risks in the high dimensional regime in the following sense:
\begin{eqnarray}\label{eq:asy_equiv}
\lim_{\substack{n \to \infty, \; p/n \to c_p\\ m/n \to c_m }} |E_{train} (\lambda) - E^\bullet_{train}(\lambda)| &=& 0,\label{eq:asy_equiv1} \\
\lim_{\substack{n \to \infty, \; p/n \to c_p\\ m/n \to c_m\; \tilde{n}/n \to \tilde{c}}} |E_{test}(\lambda) - E^\bullet_{test}(\lambda)| &=& 0,\label{eq:asy_equiv2}
\end{eqnarray}
for $\bullet \in \lbrace \lozenge, \ \square \rbrace$.

The derivation of these asymptotic equivalents is facilitated by the fact that $\hat{\theta}_\lambda$ has the explicit expression:
$$
\hat{\theta}_{\lambda} = (H H^\top + n \lambda I_m)^{-1} H Y_n = H(H^\top H + n \lambda I_n)^{-1} Y_n,
$$
where $H = h\left( \frac{WX_n^\top}{\sqrt{p}} \right) \in \mathbb{R}^{n \times m}$. To analyze the predictive risk, we also introduce the notation \textcolor{black}{for the test sample}
$$
\tilde{X}_{\tilde{n}} = \begin{pmatrix} \tilde{x}_1^\top \\ \hdots \\ \tilde{x}_{\tilde{n}}^\top \end{pmatrix} {\CM =  \tilde{\Upsilon}_x \circ \tilde{X}_{\tilde{n}}'}, \; \tilde{Y}_{\tilde n} = \begin{pmatrix} \tilde{y}_1 \\ \hdots \\ \tilde{y}_{\tilde{n}} \end{pmatrix}, \mbox{ and } \tilde{H} = h\Big(\Big\{ \frac{W\tilde{X}_{\tilde{n}}^\top}{\sqrt{p}} \Big\} \Big),
$$
\textcolor{black}{where $\tilde{X}'_{\tilde n}$ is a $\tilde n \times p$ random matrix with iid centered unit-variance entries.}
 Now, introducing the resolvent  
$$
Q(-\lambda) = (H^\top H/n + \lambda I_n)^{-1},
$$
one has that $\hat{\theta}_{\lambda} = HQ(-\lambda)Y_n/n$. The  resolvent matrix $Q(-\lambda)$ is well known in RMT as it is associated with understanding the singular values distribution of the matrix $H/\sqrt{n}$, providing critical insights into the spectral behavior of high-dimensional random matrices. As such, both $\hat{\theta}_{\lambda}$ and the associated training and predictive risks are intrinsically linked to the spectral distribution of the matrix $H^\top H/n$ via its resolvent $Q(-\lambda)$.

{\CB The asymptotic equivalents that are derived in this paper extend the results of \cite{adlam2020neural} to the case where  {\CM $X_n$ is a random matrix with a variance profile}}. {\CB One of our contributions is then to show} that, for some variance profiles, the expression of $E_{train}$ matches with the case of constant variance profiles (corresponding to iid data) dealt by \cite{adlam2020neural}. 
The derivation of   $E^\lozenge_{train}(\lambda)$ and $E^\lozenge_{test}(\lambda)$ mainly rely on the results of \cite{DaboMale} that uses the {\CB so-called traffic theory in free probability \cite{zbMATH07338275}} to prove that {\CB the expressions of $E^\lozenge_{train}(\lambda)$ and $E^\lozenge_{test}(\lambda)$ are simply obtained by replacing $H$ with its ``linear-plus-chaos" approximation $H^{\lozenge}$ (defined in Section \ref{sec:lozenge}) in the closed-form expressions of  $E_{train} (\lambda)$ and $E_{test} (\lambda)$ that are given in Section \ref{sec:fst_equiv}. Nevertheless, these first asymptotic equivalents remain expectations of random quantities that depend on the random matrices $X_n$ and $W$. {\CM To the contrary, the second set of asymptotic equivalents}  $E^\square_{train}(\lambda)$ and $E^\square_{test}(\lambda)$ only depend on the variance profiles $\Gamma_{n} = (\gamma_{ij}^2)$ and $\tilde{\Gamma}_{n} = (\tilde{\gamma}_{ij}^2)$.  The construction of this second set of asymptotic equivalents heavily relies on results related to the notion of amalgamation over the diagonal in free probability \cite{20-AOP1447} of block matrices whose \rev{blocks} are diagonal. These asymptotic equivalents are discussed in more detail in Sections \ref{sec:lozenge} and \ref{sec:square}}.
\rev{At a high level, the lozenge equivalents are obtained by replacing the nonlinear random feature matrix $H$ by its Gaussian surrogate $H^{\lozenge}$ inside the exact risk formulas. The square equivalents then replace the remaining random operator-valued resolvent terms of the lozenge model by the deterministic fixed-point solution $\mathfrak{Q}^{\square}$. In that sense, the lozenge step produces a random surrogate model, whereas the square step yields fully deterministic approximations. Although we do not introduce a separate symbol $H^{\square}$, the square quantities play exactly this deterministic-surrogate role at the level of the risk formulas.}

\subsubsection{Application to features following a mixture model}\label{sec:mixture}
 As an application, our  methodology offers a novel framework for analyzing predictors that arises from  a mixture model, that is {\CB when the data come from multiple underlying subpopulations.  Indeed, consider  latent class variables $C_1, \ldots, C_n$ (assumed to be fixed in our analysis), which determine the class membership of each  feature vector $x_i$ for  $1 \leq i \leq n$.} Within each class, the random predictor $x_i$  is then assumed to have a specific covariance structure as follows. If $x_i$ belongs to the $k$-th class among $K$ possibilities, that is $C_i = k$ for  some $1 \leq k \leq K$,  we model this predictor as
$$
x_i = S_k^{1/2} x_i',
$$
where $S_k = \text{diag}(s_{k,1}^2, \ldots, s_{k,p}^2)$ is a diagonal matrix that characterizes the covariance structure of the predictors within the $k$-th class. The variances $s_{k,j}^2$ of the predictors vary across classes, reflecting potential heterogeneity in the data.  Then, given the class labels $C_1, \ldots, C_n$, the resulting matrix of predictors $X_n$ exhibits a variance profile governed by the matrix
$$
\Gamma_{n} = (s_{C_i,j}^2) \in \mathbb{R}^{n \times p},
$$
where each entry $s_{C_i,j}^2$ corresponds to the variance of the $j$-th coordinate of the $i$-th  feature vector, determined by its class membership $C_i$.

\begin{figure}[htbp]
\begin{center}
{\subfigure[]{\includegraphics[width = 0.55\textwidth]{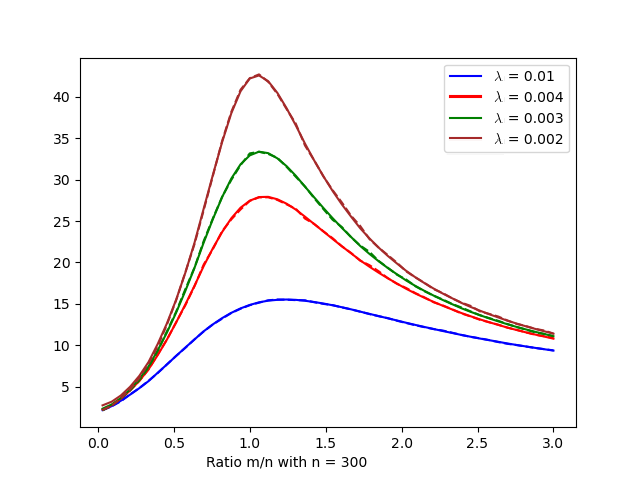}}
\label{fig:db_comparaison_const_db}}
\hfill
{\subfigure[]{\includegraphics[width =0.4\textwidth]{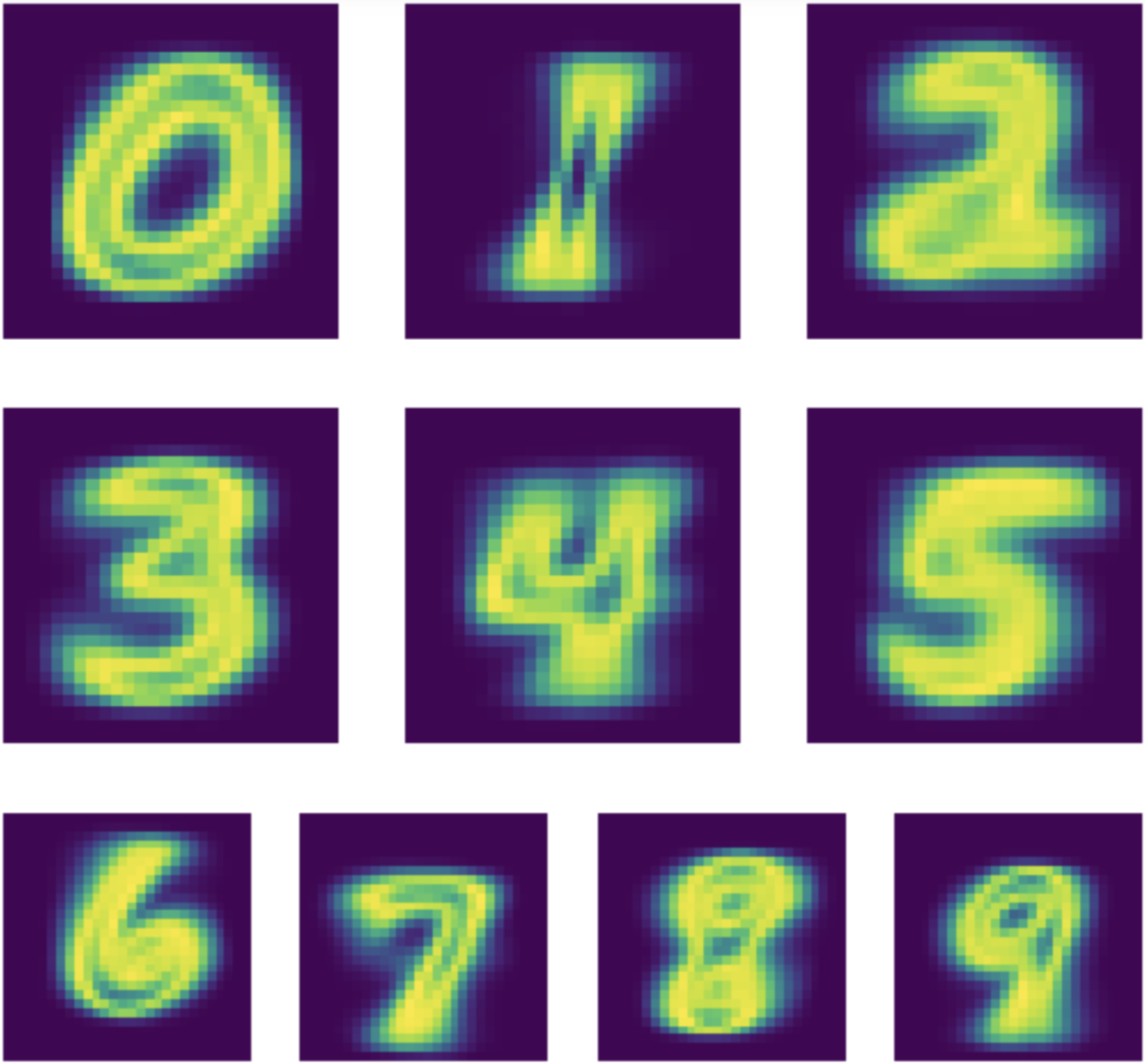}}
\label{fig:db_pw_Q_0.0005}}
\end{center}
\caption{(a) Comparison of the predictive risk $E_{test}(\lambda)$ (dashed curves) with the square asymptotic equivalent $E_{test}^{\square}(\lambda)$ (solid curves) for different values of $\lambda$ and  the ratio $m/n$. (b) Variance profile of the digits from the MNIST dataset within each class.} \label{fig:intro}
\end{figure}

To illustrate the benefits of our approach to handle the setting of predictors following such a mixture model, we have conducted numerical experiments based on the MNIST dataset \cite{lecun2010mnist}, a well-known image classification dataset comprising handwritten digits from 0 to 9. This dataset follows a 10-class mixture model where the  class $C_k$ corresponds to the digit $ k $. Each class exhibits a distinct variance profile, as illustrated in Figure \ref{fig:intro} (b) that displays   images representing the variance of each pixel of the digits within each class. These variances profiles have been computed using the full MNIST dataset made of 60000 images. It can be observed that they vary significantly from one class to another, which motivates the use of the variance profile approach. By vectorizing these images, we extract, for each class $k$, a variance profile $ S_k = \text{diag}(s_{k})$ with $s_k = (s_{k,1}^2, \ldots, s_{k,p}^2) \in \mathbb{R}^p$ of dimension $p=28 \times 28 = 784$. These diagonal matrices are normalized so that they satisfy Assumption \ref{hyp:rowsto} below. Then, we simulate data from the regression model \eqref{eq:linmod} for the variance profiles
\begin{equation} \label{eq:profileMNIST}
\Gamma_n = \begin{pmatrix} \mathbf{1}_{n_0} s_{1}^{\top}\\
\mathbf{1}_{n_0} s_{2}^{\top}\\
\vdots \\ \mathbf{1}_{n_0} s_{K}^{\top}
\end{pmatrix}
\in \RR^{n \times p} \; \mbox{ with } n = K n_0, \quad \mbox{ and } \quad \tilde{\Gamma}_n = \begin{pmatrix} \mathbf{1}_{\tilde{n}_0} s_{1}^{\top}\\
\mathbf{1}_{\tilde{n}_0} s_{2}^{\top}\\
\vdots \\ \mathbf{1}_{\tilde{n}_0} s_{K}^{\top}
\end{pmatrix}
\in \RR^{n \times p} \; \mbox{ with } \tilde{n} = K \tilde{n}_0,
\end{equation}
where $\mathbf{1}_{n_0}$ denotes the vector of length $n_0$ with all entries equal to one.  We then apply RF ridge regression with $h(x) = x^3$  for several values of  $ \lambda $. In all the numerical experiments reported in the paper, the entries of $W'$, $X_n'$ and $\beta$ are chosen as iid random variables sampled from the standard Gaussian distribution $\mathcal{N}(0,1)$.
\rev{We emphasize that MNIST is used here only as a data-driven source of class-dependent diagonal variance profiles. The simulated predictors are not raw MNIST images: they are Gaussian vectors whose coordinate-wise variances are prescribed by the matrices $S_k$. This controlled design isolates the effect of heteroscedasticity and sample-to-sample variability, which is precisely the mechanism captured by our theory. It also makes explicit the limitation of the present model: the coordinates of $X_n'$ are independent, so local pixel correlations of real images are not represented in these experiments.}

In Figure \ref{fig:intro}(a), we compare the  predictive risk $E_{test}(\lambda)$ (approximated by  a Monte Carlo approach) with the square asymptotic equivalent $E_{test}^{\square}(\lambda)$ derived from our theoretical framework, for different values of the ratio $m/n$ ranging from $0.03$ to $3$ with $n = 300$ and $\tilde{n} = 100$.  Notably, the curve corresponding to the asymptotic equivalent closely matches the one for the predictive risk, demonstrating the accuracy of our theoretical results. One can observe in Figure \ref{fig:intro}(a) the presence of a peak in the predictive risk curve whenever $m \approx n$ that is sharper as $\lambda$ goes to $0$. This peak indicates the appearance of the double descent phenomenon for small values of $\lambda$.
{\CB Therefore, the above numerical experiments show that our approach \rev{allows us to investigate} the double descent phenomenon in the context of RF regression when the data follow a mixture model, which is,  to the best of our knowledge, an aspect that has not been previously explored in the literature.}

{\CM
\subsubsection{The linear plus chaos approximation}\label{sec:linchaos}
}

The inherent non-linear randomness of the matrix $H$ induced by the activation function $h$ complicates the analysis of the spectrum of $H^\top H/n$ {\CB when the dimensions $n,p$ and $m$ tend to infinity proportionally}. {\CB For iid predictors, these complexities have already been} addressed in the literature, notably in   \cite{peche19,benigni2021eigenvalue}, which build on the {\CB fundamental} results established in \cite{NIPS2017_6857,adlam2020neural}. These studies focus on the setting where the entries of both matrices $W$ and $X_n$ are iid centered random variables with variances  $\sigma^2_w$ and  $\sigma^2_x$ respectively, effectively addressing scenarios involving iid  data. The results from \cite{benigni2021eigenvalue} reveal that the spectral distribution of $H/\sqrt{n}$ converges to a limiting distribution that coincides with the limiting spectral distribution of the {\CB matrix   $H_{lin}$, referred to as the  ``linear-plus-chaos" approximation of $H$, that is defined by:}
\begin{equation}
\frac{H_{lin}}{\sqrt{n}} = \sqrt{\theta_2(h)}\frac{W{\CB X_n^\top}}{\sqrt{np}} + \sqrt{\theta_1(h) - \theta_2(h)}\frac{Z^G}{\sqrt{n}}, \label{eq:Hlin}
\end{equation}
where the terms $\theta_1(h)$ and $\theta_2(h)$ are defined as
$$
\theta_1(h) = \mathbb{E}\left[h\left( \sigma_x \xi\right)^2\right], \quad \theta_2(h) = \mathbb{E}\left[\sigma_x h^{\prime}\left(  \sigma_x \xi\right)\right]^2,
$$
for $\xi \sim \mathcal{N}(0,1)$, and  $Z^G$ denotes an independent matrix with standard Gaussian iid  entries. The matrix $H_{lin}$ simplifies the analysis as it excludes the non-linear transformations introduced by the activation function $h$ applied to {\CB entries of the random matrix $\frac{W X^\top}{\sqrt{p}}$}. {\CB The ``linear-plus-chaos" approximation  \eqref{eq:Hlin},  also referred to in the literature}  as a Gaussian equivalence model, serves as a foundation in various studies \cite{adlam2020neural,adlam2022random,triple-descent,MeiMonta22,9931146} that analyze the training and predictive risks of RF ridge regression with predictors {\CB following the model} $x_i = \Sigma^{1/2} x_i'$. 

In the recent work \cite{DaboMale}, significant progress have been made {\CB to derive a  ``linear-plus-chaos" approximation of $H$ in the setting where} 
\begin{equation}\label{eq:WX_prime}
 W = {\CB  \Upsilon_w  \circ  W' }    \quad \text{and} \quad X_n = {\CB \Upsilon_x  \circ X_n'}  ,
\end{equation}
with matrices $\Upsilon_w$ and $\Upsilon_x$ inducing arbitrary variance profiles, and $W'$ and $X_n'$ are matrices  made of iid entries, each having  zero expectation and variance one.

Recalling that
$
H = h {\CB \left( \frac{WX_n^\top}{\sqrt{p}}   \right)},
$
the results in \cite{DaboMale}[Theorem 3.1] reveal that the matrix $H/\sqrt{n}$ can be asymptotically approximated {\CB by the following  ``linear-plus-chaos" decomposition}
\begin{equation}
\frac{H^{\lozenge}}{\sqrt{n}} = \Theta_{lin}(h) \circ \left(\frac{\mathcal W \mathcal X_n^{\top}}{\sqrt{np}}\right) + \Theta_{chaos}(h)\circ \frac{Z^G}{\sqrt{n}}. \label{eq:Hlozenge}
\end{equation}
In the above equation, $\mathcal{W}$ and $\mathcal{X}_n$ are random matrices defined as
\begin{equation}\label{eq:WX_G}
\mathcal W = {\CB \Upsilon_w \circ  W^G} \quad \text{and} \quad \mathcal X_n = {\CB  \Upsilon_x \circ X^G},
\end{equation}
where $W^G$, $X^G$, and $Z^G$ are independent matrices, each made of iid entries sampled from a standard Gaussian distribution. The decomposition \eqref{eq:Hlozenge} emphasizes the role of  the variance profiles $ \Upsilon_w^{\circ 2}$ and $\Upsilon_x^{\circ 2}$ in influencing the behavior of the matrix $H$. The matrix terms $\Theta_{lin}(h)$ and $\Theta_{chaos}(h)$ are critical to the approximation of $H$ by $H^{\lozenge}$, and they  are defined as follows
\begin{equation}
\Theta_{lin}(h) = \mathbb{E}\left( h'\left( \xi M_2 \right) \right),
\quad \mbox{and} \quad
\Theta_{chaos}(h) = \sum_{{\CB \ell }\geq 2}\frac{1}{{\CB \ell }!} \mathbb{E}\left( h^{({\CB \ell })}\left( \xi M_2 \right) \right), \label{eq:theta_matrix}
\end{equation}
where $\xi \sim \mathcal{N}(0,1)$. The quantity $M_2$ is a {\CB $m \times n$ matrix} defined as:
$$
M_2 = \sqrt[\circ]{ p^{-1}  \Upsilon_w^{\circ 2}  (\Upsilon_x^\top)^{\circ 2}},
$$
where the notation $A^{\circ 2} = A \circ A$ refers to the Hadamard product of the matrix $A$ with itself, while $\sqrt[\circ]{A} = (\sqrt{a_{ij}})$ indicates an element-wise square root of the matrix $A$. We define $\tilde{H}^{\lozenge}$ accordingly by replacing, in Equation \eqref{eq:Hlozenge}, the quantities  depending on the training data by  equivalent quantities depending on the test data and the variance profile $\tilde{\Gamma}_{n} = (\tilde{\gamma}_{ij}^2)$, that is
\begin{equation}
\frac{\tilde{H}^{\lozenge}}{\sqrt{\tilde{n}}} = \tilde{\Theta}_{lin}(h) \circ \left(\frac{\mathcal W \tilde{\mathcal X_n}^{\top}}{\sqrt{\tilde{n}p}}\right) + \tilde{\Theta}_{chaos}(h)\circ \frac{\tilde{Z}^G}{\sqrt{\tilde{n}}}. \label{eq:tilde_Hlozenge}
\end{equation}

The results from \cite{DaboMale} indicate that {\CB $\frac{H}{\sqrt{n}}$ and  $\frac{H^{\lozenge}}{\sqrt{n}}$ have asymptotically the same spectral distribution}. Furthermore, their asymptotic behavior under polynomial transformations are identical {\CB thanks to the following proposition that is a direct consequence of   \cite{DaboMale}[Theorem 3.1].} 
\begin{prop}\label{prop:H-H_lin}
Under Assumptions \ref{hyp:moments-analytical}, \ref{hyp:h-analytical}, \ref{hyp:bounded_profile}, {\CB given in Section \ref{sec:notations}}, and for any polynomial $P$ in the matrices and their transpose, the following holds true
\begin{equation}\label{thm:H-H_lin+tilde}
\lim_{\substack{n \to \infty, \; p/n \to c_p\\ m/n \to c_m }}     \frac{1}{n}\Tr \left[\mathbb{E}\left[P\left(\frac{H}{\sqrt{n}},\frac{W}{\sqrt{n}},\frac{X_n}{\sqrt{p}}\right) - P\left(\frac{H^{\lozenge}}{\sqrt{n}},\frac{\mathcal W}{\sqrt{n}},\frac{\mathcal X_n}{\sqrt{p}}\right)\right]\right] \rightarrow 0.
\end{equation}
\end{prop}

Proposition \ref{prop:H-H_lin} means that, as  {\CB the dimensions $n,p$ and $m$ tend to infinity proportionally}, the difference in the expected values of a matrix polynomial evaluated either at $\frac{H}{\sqrt{n}}$ or $\frac{H^{\lozenge}}{\sqrt{n}}$ (along with the other random matrices associated to their expressions)  converges to zero in the trace sense. This result remains valid if $H$, $H^{\lozenge}$, and $X_n$ are replaced by quantities associated to the test data, that is by $\tilde{H}$, $\tilde{H}^{\lozenge}$, and $\tilde{\mathcal{X}_n}$.

\subsubsection{Lozenge equivalents}\label{sec:lozenge}

{\CM To derive the asymptotic equivalents of the training and predictive risk, we restrict the analysis, in the rest of the paper,} to the case where the matrix $W$ has a constant variance profile that is
\begin{eqnarray}\label{eq:profil_w}
\Upsilon_w = \sigma_w \begin{pmatrix} 1 & \cdots & 1\\ \vdots & & \vdots \\ 1 & \cdots & 1\end{pmatrix}, \quad \text{with} \quad \sigma_w > 0.
\end{eqnarray}
\rev{Without loss of generality,} we shall assume that $\sigma_w=1$ as a constant variance of the entries of $W$ only represents a scaling effect on the variance profile of $X_n$ when considering the product $W X_n^\top$. 
In this setting, the matrices $\Theta_{lin}(h)$ and $\Theta_{chaos}(h)$ {\CB defined by \eqref{eq:theta_matrix}} reduce to {\CB rank one} matrices, {\CB and} this leads to simpler expressions for $H^{\lozenge}$ and $\tilde{H}^{\lozenge}$. Indeed, {\CM in this case and as discussed in  \cite{DaboMale}[Example 3.2]},   it follows that the ``linear-plus-chaos" approximation becomes:
\begin{eqnarray}\label{eq:H-H_lin-diag}
H^{\lozenge} = \frac{\mathcal W \mathcal X_n^\top D_{lin}(h)}{\sqrt{p}} + Z^G D_{chaos}(h) \mbox{, }
\tilde{H}^{\lozenge} = \frac{\mathcal W \tilde{\mathcal X}_{\tilde{n}}^\top \tilde{D}_{lin}(h)}{\sqrt{p}} + \tilde{Z}^G \tilde{D}_{chaos}(h), \label{eq:Htilde-Htilde_lin-diag}
\end{eqnarray}
where $D_{lin}(h)$ and $D_{chaos}(h)$ are defined as follows (and similarly for $\tilde{D}_{lin}(h)$ and $\tilde{D}_{chaos}(h)$):
$$
D_{lin}(h) = \mathbb{E}\left[h'\left(\xi D_{2}\right)\right], \quad D_{chaos}(h) = \sum_{{\CB \ell}\geq 2}\frac{1}{{\CB \ell}!} \mathbb{E}\left[h^{({\CB \ell})}\left(\xi D_{2}\right)\right],
$$
with $\xi \sim \mathcal{N}(0,1)$ and $D_{2} =  \mathrm{deg}(p^{-1}  \Upsilon_x^{\circ 2})^{1/2}$ 
, where, for a matrix $A$, the notation $\mathrm{deg}(A)$  denotes the diagonal matrix whose $k$-diagonal element is the sum of the entries of the $k$-row of $A$.

 Then, as shown in  Section \ref{sec:fst_equiv}, the training error $E_{train}$ and the predictive error $E_{test}$ can be expressed as expectations involving traces of matrix polynomials depending on $H, W, X_n, \tilde{H}, \tilde{X}$. {\CB Therefore, using Proposition \ref{prop:H-H_lin} {\CM and under the assumption \eqref{eq:profil_w} that $W$ has a constant variance profile, we  prove that}  the asymptotic equivalents $E^\lozenge_{train}$ and $E^\lozenge_{test}$  of $E_{train}$ and $E_{test}$  are obtained} by substituting the quantities $H, W, X_n, \tilde{H}, \tilde{X}_n$ with $H^{\lozenge}, \mathcal{W}, \mathcal{X}_n, \tilde{H}^{\lozenge}, \tilde{\mathcal{X}}_{\tilde{n}}$ in the  expressions for $E_{train}$ and $E_{test}$ given in Section \ref{sec:fst_equiv}. 

\subsubsection{Square equivalents}\label{sec:square}
{\CM When $W$ has a constant variance profile, it can be seen from Equation \eqref{eq:H-H_lin-diag} that} $H^\lozenge$ {\CB becomes a non-commutative matrix polynomial in the variables $ \mathcal X_n, \mathcal W, D_{lin}(h), D_{chaos}(h)$ and $Z^G$}. Therefore, using the so-called linearization trick \cite{mingo2017free}[Chapter 10] in free probability, we are able to obtain a deterministic equivalent of the spectral distribution of the random matrix  $H^\lozenge$, which leads to the construction of the second set of asymptotic equivalents of the training and predictive risk in RF regression. This  linearization trick is detailed in Section \ref{sec:proof_traffic}.

{\CB For specific classes of variance profile and choices of activation function $h$,  these asymptotic equivalents have explicit expressions that are related to known results for RF regression in the case of iid predictors.  For example, let us consider the following assumption on the variance profile matrices $\Gamma_{n}$ and $\tilde{\Gamma}_{n}$.
\begin{hyp}\label{hyp:rowsto}
There exists a  constant $s > 0$ such that the rows of the variance profile of $X_n$ and $\tilde{X}_{\tilde{n}}$   satisfy
$$
\frac{1}{p} \sum_{j=1}^{p}  \gamma_{ij}^2 = s^2 \; \mbox{ and } \; \frac{1}{p} \sum_{j=1}^{p}  \tilde{\gamma}_{ij}^2 = s^2 \; \mbox{ for all } 1 \leq i \leq n.
$$
\end{hyp}
In other words, Assumption \ref{hyp:rowsto} means that the matrices $s^{-2}\Gamma_{n}$ and $s^{-2} \tilde{\Gamma}_{n}$ are row stochastic.} Under Assumption \ref{hyp:rowsto}, 
 the expression of {\CB the ``linear-plus-chaos" approximations} $H^{\lozenge}$ and $\tilde{H}^{\lozenge}$ becomes even simpler, as Equations \eqref{eq:H-H_lin-diag} and \eqref{eq:Htilde-Htilde_lin-diag} simplify to: 
 \begin{eqnarray}\label{eq:H-H_lin-constant}
H^{\lozenge} &=& \theta_{lin}(h)\frac{\mathcal W\mathcal X_n^\top}{\sqrt{p}} + \theta_{chaos}(h) Z^G \\
\tilde{H}^{\lozenge} &=& \theta_{lin}(h)\frac{\mathcal W\tilde{\mathcal X}_{\tilde{n}}^\top}{\sqrt{p}} + \theta_{chaos}(h) \tilde{Z}^G,
\end{eqnarray}
with $\theta_{lin}(h) = \bigg( \mathbb{E}[h'( \xi s ) ] \bigg)$ and $\theta_{chaos}(h) = \sum_{{\CB \ell}\geq 2}\frac 1 {{\CB \ell}!} \bigg( \mathbb{E}[   h^{({\CB \ell})}( \xi s ) ]  \bigg)$, for $\xi \sim \mathcal{N}(0,1)$.

{\CB Then, under the supplementary assumption that  $\theta_{lin}(h) = 0$ (and still those of Proposition \ref{prop:H-H_lin}), we are able to provide explicit expressions for $E_{train}^\square$ and $E_{test}^\square$ as follows}, 
 \begin{eqnarray}
E^\square_{train}(\lambda) &=& \lambda^2\alpha^2s^2m_n'(-\lambda)  + \lambda^2\sigma^2 m_n'(-\lambda), \label{eq:Etrain_explicit} \\
E^\square_{test}(\lambda) &=& \sigma^2 + \alpha^2s^2 + \theta_{chaos}^2(h)(\alpha^2s^2+\sigma^2)(m_n(-\lambda) - \lambda m_n'(-\lambda))  \label{eq:Etest_explicit},
\end{eqnarray}
where
\begin{equation} \label{eq:m_n}
\textcolor{black}{m_n(-\lambda) = \frac{\theta_{chaos}^2(h)(\varphi_m-1)+\lambda-\sqrt{\left(\lambda+\theta_{chaos}^2(\varphi_m+1)\right)^2-4 \varphi_m \theta_{chaos}^4}}{2 \varphi_m \lambda \theta_{chaos}^2(h)}}
\end{equation}
denotes  the Stieltjes transform  of the Marchenko-Pastur distribution with parameter $\varphi_m = \lim_{n,m \rightarrow +\infty} \frac{n}{m}$ \textcolor{black}{\cite{bai2010spectral}(Lemma 3.11)},   and $m_n'$ denotes its derivative. 
{\CB The expression \eqref{eq:Etrain_explicit} of $E_{train}^\square$ matches with the expression of the analogous asymptotic equivalent in \cite{adlam2020neural} in the case where $X_n$ has a constant variance profile. The expression  \eqref{eq:Etest_explicit}  of $E_{test}^\square$ being more complicated, it is  difficult to compare it with the results of \cite{adlam2020neural} that are focused on a slightly different notion of prediction error based on generalized cross-validation.

When  $\theta_{lin}(h) \neq 0$, we anticipate that the expressions of $E_{train}^\square$ and $E_{test}^\square$ are not explicit. However, in this case and  under Assumption \ref{hyp:rowsto},  we are able to provide accurate numerical approximation of these second set of asymptotic equivalents through the resolution of a fixed-point equation that is stated in Section \ref{sec:snd_equiv}, see Theorem \ref{thm:adaptation} and Equation \eqref{eq:fixedpoint}. }

\subsection{Organization of the paper}
The notation needed to derive the results of this paper are presented in Section \ref{sec:notations}, along with the  assumptions  required for the developments of our contributions. The main results are presented and discussed in Section \ref{sec:results}. {\CB We first describe the  asymptotic equivalents $E_{train}^\lozenge$ and $E_{test}^\lozenge$. Then, we introduce free probability results that enable us to provide the second set of asymptotic equivalents,  namely $E_{train}^\square$ and $E_{test}^\square$.  Section \ref{sec:proof_main} is dedicated to {\CB the proof of} these main results. The proof of the derivation of the expressions of $E_{train}^\lozenge$ and $E_{test}^\lozenge$ based on the results of \cite{DaboMale}  is given in Section \ref{sec:proof_traffic}.  Then, we prove in Section \ref{sec:proof_free} and Section \ref{sec:proof_square} that $E_{train}^\square$ and $E_{test}^\square$ are relevant asymptotic equivalents of the  training and predictive risks  using the linearization trick and an adaptation of a free probability result on deformed random matrices with a variance profile from \cite{bigotmale}}. We conduct numerical experiments in Section \ref{sec:num}  that witness the sharpness of $E_{train}^\square$ and $E_{test}^\square$ and the appearance of the double descent phenomenon {\CB in various setting of RF regression from non-identically distributed data with a variance profile}.  Finally, auxiliary results and proofs needed in Section \ref{sec:proof_main}  are gathered in a technical Appendix.

\section{{\CB Notation and preliminaries}}\label{sec:notations}

We let $n,m,p$ be positive integers. We denote $N = n+m+2p$, $c_n = \sqrt{\frac{N}{n}}$, $c_p = \sqrt{\frac{N}{p}}$ and $c_m = \sqrt{\frac{N}{m}}$,  and we assume that there exist constants $\varphi_m,\varphi_p >0$ for which 
$$
 \lim_{n,m,p\rightarrow +\infty} \frac{n}{p} = \varphi_p \quad \and \quad \lim_{n,m,p\rightarrow +\infty} \frac{n}{m} = \varphi_m.
$$
{\CM In the paper, whenever it is stated that $n$ tends to infinity or that it is sufficiently large, it is also understood that $m$ and $p$ grow according to the proportional asymptotic regime described above.}

Let $A = (A_{i,j})$ be a matrix in $\mathbb{C}^{n_1\times n_2}$. We define its operator (or spectral) norm $\| A \|$, its Frobenius norm $\|A\|_F$, its normalized Frobenius norm $\|A\|_{F,n_1}$, {\CM and its max-norm $\|A\|_{\infty}$ as follows}
$$
\| A \|^2 = \sup_{\substack{x\in\mathbbm{C}^{n_2},\\ \|x\|_2 = 1}}\|Ax\|_2^2, \quad \| A \|_F^2 = \Tr[AA^\ast], \quad   \| A \|_{F,n_1}^2 = \frac{1}{n_1}\Tr[AA^\ast],  \quad \mbox{and } {\CM \| A \|_{\infty} = \max_{1 \leq i \leq n_1 ,1 \leq j \leq n_2} |A_{i,j}|}
$$
where $ \| (x_1,...,x_{n_2})\|_2^2 = \sum_{k=1}^{n_2}|x_k|^2$, and $\Tr$ denotes the trace. These norms are equivalents, namely for any matrix $A$, one has that  
$
\| A \|_{F,n_1}^2 \leq \| A \|^2 \leq \| A \|_F^2.
$
We denote by $\Re e(A)$ and $\Im m (A)$, the real and imaginary parts of $A$, defined as 
$
\Re e(A) = \frac{1}{2}(A+A^\ast),  \Im m (A) = \frac{1}{2\mathbf{i}}(A-A^\ast).
$
For any Hermitian matrix $A$, we write $A >0$ (respectively $A \geq 0$) whenever is positive definite (respectively positive semi-definite). We denote by $D_N^+(\mathbb{C})$ the set of diagonal matrices $\Lambda \in \mathbb{C}^{N\times N}$, such that $\Im m (\Lambda) > 0$.
\begin{table}[ht]
\centering
\caption{\rev{Main notation used throughout the paper.}}
\begin{tabular}{p{0.30\textwidth}p{0.60\textwidth}}
\hline
{\color{black}
$\Upsilon_x=(\gamma_{ij}^x)$} & {\color{black}Variance profile of the entries of the data matrix $X$.} \\
{\color{black}
$\Upsilon_w=(\gamma_{ij}^w)$} & {\color{black}Variance profile of the entries of the weight matrix $W$.} \\
{\color{black}
$\gamma_{ij}^x$} & {\color{black}Variance parameter associated with the entry $(i,j)$ of $X$.} \\
{\color{black}
$\gamma_{ij}^w$} & {\color{black}Variance parameter associated with the entry $(i,j)$ of $W$.} \\
\rev{$X_n=\Upsilon_x\circ X_n'$} & \rev{training design matrix with row-wise diagonal variance profile}\\
\rev{$\tilde X_{\tilde n}=\tilde\Upsilon_x\circ \tilde X'_{\tilde n}$} & \rev{test design matrix}\\
\rev{$H=h(WX_n^\top/\sqrt{p})$} & \rev{random feature matrix}\\
\rev{$H^{\lozenge}$} & \rev{linear-plus-chaos / Gaussian surrogate of $H$}\\
\rev{$Q(-\lambda)$} & \rev{resolvent $(H^\top H/n+\lambda I_n)^{-1}$}\\
\rev{$Q^\lozenge(-\lambda)$} & \rev{resolvent $(H^{\lozenge\top} H^\lozenge/n+\lambda I_n)^{-1}$}\\
\rev{$\mathfrak{Q}(\Lambda)$} & \rev{operator-valued resolvent of the linearization matrix $L$}\\
\rev{$E^{\lozenge}_{train},E^{\lozenge}_{test}$} & \rev{random-surrogate asymptotic equivalents obtained from $H^{\lozenge}$}\\
\rev{$E^{\square}_{train},E^{\square}_{test}$} & \rev{deterministic asymptotic equivalents obtained from the fixed-point equation}\\
{\color{black}
$\Theta_{lin}(h), \Theta_{chaos}(h)$} & {\color{black}Matrix coefficient defined in \eqref{eq:theta_matrix}, associated with the linear part of $H^{\lozenge}$.} \\
{\color{black}
$D_{lin}(h),D_{chaos}(h)$} & {\color{black}Diagonal matrix coefficient appearing in \eqref{eq:H-H_lin-diag}, associated with the linear part of $H^{\lozenge}$.} \\
{\color{black}
$\theta_{lin}(h),\theta_{chaos}(h)$} & {\color{black}Scalar coefficient appearing in \eqref{eq:H-H_lin-constant}, associated with the linear part of $H^{\lozenge}$.} \\
\hline
\end{tabular}
\end{table}

\begin{lem}\label{lem:Im_pos}[\cite{haagerup2005new}[Lemma 3.1], \cite{bigotmale}[Lemma 5.1]]
Let $A$ in $M_N(\mathbb{C})$ such that $\Im m(A) > 0$. Then $A$ is invertible and
$
\|A^{-1}\| \leq \|(\Im m(A))^{-1}\|.
$
\end{lem}

\begin{lem}\label{lem:fro_sub}
Let $n_1,n_2$ and $n_3$ be positive integers and $A\in \mathbbm{C}^{n_1,n_2}$ $B\in\mathbbm{C}^{n_2,n_3}$ be complex random matrices, then one has that
$
\|AB\|_{F,n_1} \leq \|A\|_{F,n_1}\|B\|.
$
\end{lem}

{\CM
Finally, we have the following lemma that follows from \cite{Horn_Johnson_1991}[Theorem 5.6.2].

\begin{lem} \label{lem:bound Hadamard}
Let $A \in \mathbbm{C}^{n_1,n_2}$  that admits the factorization $A = U^\ast V$ with $U \in \mathbbm{C}^{r,n_1}$ and $V \in \mathbbm{C}^{r,n_2}$ where $U$ and $V$ are matrices such that the Euclidean norm of their columns is bounded by a constant $c_{0} > 0$. Then, $
\|A \circ B \| \leq c_{0}^2 \| B \|
$ for any $B \in \mathbbm{C}^{n_1,n_2}$.
\end{lem}

}

{\CB Throughout the paper, we shall make the following assumptions}.
\begin{hyp}\label{hyp:moments-analytical}
There exist constants $\vartheta_x>0$, $\vartheta_w>0$ and $\chi>1$ such that for all $i,j$
$$
\mathbb{P} \big( |X_n'(i,j)| \geq t  \big) \leq e^{-\vartheta_x t^\chi} \quad \mathrm{and} \quad \mathbb{P} \big( |{\CB W'}(i,j)| \geq t  \big) \leq e^{-\vartheta_w t^\chi}.
$$
\end{hyp}

\begin{hyp}\label{hyp:h-analytical}
 The activation function $h:\mathbb R \to \mathbb R$ is an odd polynomial.
\end{hyp}
\rev{In this paper, we restrict the theoretical analysis to odd polynomial activation functions. This assumption is mainly technical: it is the regime covered by the Gaussian-equivalence theorem of \cite{DaboMale}, and the corresponding Hermite expansion is finite, which yields an explicit linear-plus-chaos approximation. The odd-polynomial assumption should therefore be viewed as inherited from the currently available approximation theorem rather than as intrinsic to the phenomenon under study. In the i.i.d.\ setting, namely when $\Gamma_x$ and $\Gamma_w$ are constant variance profiles ($\gamma_{ij}^x=\sigma_x$ and $\gamma_{ij}^w=\sigma_w$), \cite{benigni2021eigenvalue} established Gaussian equivalence for polynomial activations and then extended the result to smooth functions by a density argument. This suggests that the polynomial restriction may be relaxed beyond the present framework. We are currently investigating such an extension to odd smooth activations in the general variance-profile setting, and our numerical experiments support its plausibility. Indeed, the experiments reported in Section~\ref{sec:num} show that the deterministic equivalents remain accurate for several non-polynomial activations, including $\tanh$, ReLU, and $|x|$, which suggests robustness beyond the scope of the current proof strategy.}

{\CM
\begin{hyp} \label{hyp:bounded_profile}
There exists a constant $\gamma > 0$   such that, for all $n,\tilde{n},m,p\geq 1$, the variance profiles of $X_n$ and $\tilde{X}_{\tilde{n}}$ admit the factorizations $\Upsilon_x = U_x^\ast V_x$ and $\tilde{\Upsilon}_x = \tilde{U}_x^\ast \tilde{V}_x$ where the columns of the matrices $U_x \in \RR^{r \times n}$,  $V_x \in \RR^{r \times p}$, $\tilde{U}_x \in \RR^{\tilde{r} \times \tilde{n}}$,  $\tilde{V}_x \in \RR^{\tilde{r} \times p}$ have their Euclidean norm that is bounded by $\sqrt{\gamma}$.
\end{hyp}
Note that Assumption \ref{hyp:bounded_profile} implies, by Cauchy-Schwarz inequality, that the entries of the variance profiles of $X_n$ and $\tilde{X}_{\tilde{n}}$ are bounded by $\gamma > 0$, that is
$
\| \Upsilon_x\|_{\infty} \leq \gamma
$
and
$ \|  \tilde{\Upsilon}_x\|_{\infty}   \leq \gamma.
$
}

{\CM
 \begin{hyp}\label{hyp:h-bounded}
 There exists a constant $C> 0$,   not depending on $n, \tilde{n},p$ and $m$, such that, with probability one,
  \begin{equation} \label{eq:id-bounded}
 \EE\left[\left\|   \frac{WX_n^\top}{\sqrt{n} \sqrt{p}}   \right\|^k\right] \leq C^k,  \; \EE\left[\left\|    \frac{W\tilde{X}_{\tilde{n}}^\top}{\sqrt{\tilde{n}}\sqrt{p}}  \right\|^k\right] \leq C^k,
\end{equation}
and
 \begin{equation}  \label{eq:h-bounded}
\EE\left[\left\|  \frac{1}{\sqrt{n}}h\Big(\Big\{ \frac{WX_n^\top}{\sqrt{p}} \Big\} \Big) \right\|^k\right] \leq C^k,  \; \EE\left[\left\|  \frac{1}{\sqrt{\tilde{n}}}h\Big(\Big\{ \frac{W\tilde{X}_{\tilde{n}}^\top}{\sqrt{p}} \Big\} \Big) \right\|^k\right] \leq C^k,
\end{equation}
 for all sufficiently large values of $n,\tilde{n},p$ and $m$.
 \end{hyp}
 
 \begin{rem} \label{rem:condsimple}
 We  may give simple conditions ensuring that Inequality \eqref{eq:id-bounded} in Assumption \ref{hyp:h-bounded}  holds  as follows.  Using Lemma \ref{lem:bound Hadamard}, Assumption \ref{hyp:bounded_profile} and the facts that $X_n =   \Upsilon_x  \circ X_n' $ and  $\tilde{X}_{\tilde{n}} =   \tilde{\Upsilon}_x  \circ \tilde{X}'_{\tilde{n}}$, we have that 
$$
\quad \left\| \frac{ X_n }{\sqrt{n}}  \right\|   \leq \gamma \left\| \frac{ X_n' }{\sqrt{n}} \right\| \quad \mbox{ and } \quad \left\| \frac{ \tilde{X}_{\tilde{n}} }{\sqrt{\tilde{n}}}  \right\|   \leq \gamma \left\| \frac{ \tilde{X}'_{\tilde n} }{\sqrt{\tilde{n}}} \right\|.
$$
Hence, using the above upper bounds, we finally obtain that
$$
\left\|   \frac{WX_n^\top}{\sqrt{n} \sqrt{p}}   \right\| \leq \gamma \left\|   \frac{W}{\sqrt{p}}   \right\|  \left\| \frac{ X'_n }{\sqrt{n}} \right\| \quad \mbox{ and } \quad \left\|   \frac{W\tilde{X}_{\tilde{n}}^\top}{\sqrt{n} \sqrt{p}}   \right\| \leq \gamma \left\|   \frac{W}{\sqrt{p}}   \right\|  \left\| \frac{ \tilde{X}'_{\tilde n} }{\sqrt{\tilde{n}}} \right\| 
$$

Consequently, provided the distributions of the entries of $W$, $X'_n$ and $\tilde{X}'$ have a finite fourth-moment, Inequality \eqref{eq:id-bounded} in Assumption \ref{hyp:h-bounded} is  satisfied since, by \cite{BS98}[Theorem 1.1] and \cite{vershynin2020high}[Theorem 4.4.3], it is known that,   $ \EE\mathbb[\|\frac{W}{\sqrt{p}}\|] $ and $ \EE\mathbb[\|\frac{ X'_n }{\sqrt{n}}\|]$ (resp.\ $ \EE\mathbb[\|\frac{ \tilde{X}' }{\sqrt{\tilde{n}}}\|]$) are bounded. 
\end{rem}
 }
 
 {\CB In the rest of this section, we introduce notation and tools to apply the linearization trick \cite{mingo2017free}[Chapter 10] in free probability to obtain a deterministic equivalent of the spectral distribution of the random matrix  $H^\lozenge$.}
Let $M\in \mathbbm{C}^{N\times N}$ be a symmetric matrix assumed to be decomposed into $4\times 4$ blocks as follows:
$$
M = \begin{pmatrix}
M_{11} & M_{12} & M_{13} & M_{14}\\
M_{21} & M_{22} & M_{23} & M_{24}\\
M_{31} & M_{32} & M_{33} & M_{34}\\
M_{41} & M_{42} & M_{43} & M_{44}
\end{pmatrix}.
$$
 For $1\leq I,J \leq 4$ we denote by   $M_{IJ}$ the $IJ$ sub-block of $M$. Recalling that $N = n+m+2p$, we assume  that
\begin{itemize}
    \item[-] for $1\leq K \leq 4$, the sub-block $M_{1K}$ has $n$ columns,
    \item[-] for $1\leq K \leq 4$, the sub-block $M_{2K}$ has $m$ rows,
    \item[-] for $1\leq K \leq 4$, the sub-block $M_{3K}$ and $M_{4K}$ have $p$ rows.
\end{itemize}

We define the map $id\otimes \Delta : M \mapsto id\otimes \Delta [M] = (\Delta[M_{IJ}])_{1\leq I,J \leq 4}$, where $\Delta[M_{IJ}] = \diag\limits_{k}(M_{IJ}(k,k))$ if $M_{IJ}$ is a square matrix and $\Delta[M_{IJ}] = 0$  if $M_{IJ}$ is not square. Under the assumptions above, one has that
$$
id\otimes \Delta [M] = \begin{pmatrix}
\Delta[M_{11}] & 0 & 0 & 0\\
0 & \Delta[M_{22}] & 0 & 0\\
0 & 0 & \Delta[M_{33}] & \Delta[M_{34}]\\
0 & 0 & \Delta[M_{43}] & \Delta[M_{44}]
\end{pmatrix}.
$$
Let us define the following blocks matrices:
\begin{equation}
L = \begin{pmatrix} 0 & \frac{D_{chaos}(h) Z^{G\top}}{\sqrt{n}} & -\frac{\mathcal X D_{lin}(h)}{\sqrt{n}} & 0 \\
\frac{Z^G D_{chaos}(h)}{\sqrt{n}} & -I_m & 0 & -\frac{\mathcal W}{\sqrt{p}} \\
-\frac{D_{lin}(h) \mathcal X^\top }{\sqrt{n}} & 0  & 0 &  -I_p \\
0 & -\frac{\mathcal W^\top}{\sqrt{p}} & -I_p & 0 \label{eq:L}
\end{pmatrix},
\end{equation}
\begin{equation} \label{eq:CN}
C_N =   -I_m\otimes E_{22} -I_p\otimes( E_{34} + E_{43}) = \begin{pmatrix} 0 & 0 & 0 & 0 \\
0 & -I_m & 0 & 0 \\
0 & 0  & 0 &  -I_p \\
0 & 0 & -I_p & 0
\end{pmatrix}  \in \mathbb{R}^{N\times N},
\end{equation}
{\CB with $E_{kl} = (\delta_{uk}\delta_{vl})_{1\leq u,v\leq N}$, for $1\leq k,l\leq 4$,} and
\begin{equation} \label{eq:varprofileU}
\Upsilon_L =(\gamma_{ij}^{(L)}) =  \begin{pmatrix} 0 & c_nD_{chaos}(h) \mathbf{1}_{n\times m} &  c_n\Upsilon_x D_{lin}(h) & 0 \\
c_n\mathbf{1}_{m\times n} D_{chaos}(h) & 0 & 0 & c_p {\CM \Upsilon_w}  \\
c_nD_{lin}(h) \Upsilon_x^\top  & 0  & 0 &  0 \\
0 & c_p {\CM \Upsilon_w^\top}  & 0 & 0
\end{pmatrix},
\end{equation}
with $\mathbf{1}_{n_1\times n_2}$  denoting the $n_1 \times n_2$ matrix whose entries are all equal to $1$.

{\CB The matrix $L$ arises from the linearization of the ``linear-plus-chaos" approximation  $H^{\lozenge}$  of the matrix $H$ when  $W$ has a constant variance profile, and we refer to Section \ref{sec:proof_free} for  further details on this trick. Since $\mathcal W = {\CB \Upsilon_w \circ  W^G}$ and $\mathcal X = {\CB  \Upsilon_x \circ X^G}$, a key relation between $L, C_N$ and $\Upsilon_L =(\gamma_{ij}^{(L)})$ is
\begin{equation}\label{eq:key}
L = \Upsilon_L \circ A_N + C_N,
\end{equation}
where $A_N = (a_{ij}/\sqrt{N})$ is a GOE matrix (that is $A_N$ is symmetric and its coefficients $(a_{ij})_{i\leq j}$ form a sequence of iid standard real Gaussian variables). Therefore, the matrix $L$ is interpreted as an additive deformation by $C_N$ of a GOE matrix with variance profile $ \Upsilon_L^{o2}$ which is the key point to derive the second set of asymptotic equivalents.
}

We finally define the operator-valued resolvent of $L$ as follows 
\begin{eqnarray}\label{eq:op_resolvent}
\mathfrak{Q}(\Lambda) = (L - \Lambda)^{-1},
\end{eqnarray}
for every diagonal matrix $\Lambda\in \mathbb{C}^{N\times N}$ such that $L-\Lambda$ is invertible.


\begin{rem}[\rev{Interpretation of the assumptions}]
 \rev{Assumption \ref{hyp:moments-analytical} is a tail condition used to invoke the Gaussian-equivalence theorem of \cite{DaboMale}; the parameters $\vartheta_x$, $\vartheta_w$, and $\chi$ control concentration and moment growth, but they do not enter explicitly into the final deterministic equivalents. Assumption \ref{hyp:h-analytical} is the technical price paid for the currently available linear-plus-chaos approximation. Assumption \ref{hyp:bounded_profile} ensures that the variance profiles remain uniformly controlled and is used repeatedly to bound Hadamard products and operator norms. Assumption \ref{hyp:h-bounded} guarantees that the random feature matrices and the associated Gram matrices stay in a regime where resolvent manipulations and power-series arguments are legitimate. 

}
\end{rem}

\section{Main results}\label{sec:results}

 In this section, we detail the main contributions of the paper on the derivation of asymptotic equivalents for the training and  predictive risks.

\subsection{ Lozenge asymptotic equivalents}\label{sec:fst_equiv}

In order to study $E_{train}$ and $E_{test}$, we provide more tractable expressions than \eqref{eq:train} and \eqref{eq:test} in the following lemma {\CB whose proof is postponed to the Appendix}.

\begin{lem}\label{lem:risk_expressions}
Let $Q'(z)$ be the derivative of $Q(z)$  with respect to $z\in \mathbbm{C}\setminus \mathbb{R}^+$. One can express $E_{train}$ and $E_{test}$ as follows
\begin{eqnarray*}
E_{train}(\lambda) &=& \frac{\lambda^2\alpha^2}{n}\mathrm{Tr}\Bigg[\mathbb{E}\left[Q^{  '}(-\lambda)\frac{  X_n   X_n^\top}{p}\right]\Bigg] + \frac{\lambda^2\sigma^2}{n}\mathrm{Tr}\Bigg[\mathbb{E}\left[Q^{  '}(-\lambda)\right]\Bigg],
\end{eqnarray*}
\begin{eqnarray*}
E_{test}(\lambda)   &=& \sigma^2 + \frac{\alpha^2}{p\tilde n}\Tr\Bigg[\mathrm{deg}(\tilde{\Gamma}_{\tilde{n}})\Bigg] + \frac{\alpha^2}{\tilde n}\Tr\Bigg[\mathbb{E}\left[\frac{1}{pn^2}\tilde{H}^{ \top} H Q(-\lambda)  X_n  X_n^\top Q(-\lambda)H^{ \top} \tilde{H}\right]\Bigg]\\
                    & &  + \frac{\sigma^2}{\tilde n}\Tr\Bigg[\mathbb{E}\left[\frac{1}{n^2}\tilde{H}^{ \top} H Q'(-\lambda)H^{ \top} \tilde{H} \right]\Bigg]-\frac{2\alpha^2}{\tilde{n}np}\Tr\Bigg[\mathbb{E}\left[ \tilde{X}_{\tilde{n}}   X_n^\top Q(-\lambda) H^{ \top} \tilde{H}\right]\Bigg].
\end{eqnarray*}
\end{lem}

Note that $E_{train}$ and $E_{test}$ are now expressed  {\CB as matrix polynomials}  involving $H,W,X$ and $Q(-\lambda)$. The {\CB results from  \cite{DaboMale} allow}  us to state that replacing $H,W,X$ and $Q(-\lambda)$ by $H^\lozenge,\mathcal W$, $\mathcal X$ and $Q^\lozenge(-\lambda) = (H^{\lozenge\top} H^\lozenge/n + \lambda I_n)^{-1}$ in $E_{train}$ and $E_{test}$  does not change their limits when $n,m,p$ tends to $+\infty$.  This gives rise to the emergence of a first set of asymptotic equivalents for $E_{train}$ and $E_{test}$.

\begin{thm}\label{thm:equiv_risk}
Under Assumptions \ref{hyp:moments-analytical} to \ref{hyp:h-bounded}, one has that 
\begin{eqnarray*}
\lim_{\substack{n \to \infty, \; p/n \to c_p\\ m/n \to c_m}} |E_{train} (\lambda) - E^\lozenge_{train}(\lambda)| = 0  & \mathrm{and} & \lim_{\substack{n \to \infty, \; p/n \to c_p\\ m/n \to c_m \; \tilde{n}/n \to \tilde{c}}} |E_{test}(\lambda) - E^\lozenge_{test}(\lambda)| = 0,
\end{eqnarray*}
with
\begin{eqnarray*}
E^\lozenge_{train}(\lambda) &=& \frac{\lambda^2\alpha^2}{n}\mathrm{Tr}\Bigg[\mathbb{E}\left[Q^{\lozenge '}(-\lambda)\frac{\mathcal X_n \mathcal X_n^\top}{p}\right]\Bigg] + \frac{\lambda^2\sigma^2}{n}\mathrm{Tr}\Bigg[\mathbb{E}\left[Q^{\lozenge '}(-\lambda)\right]\Bigg],
\end{eqnarray*}
\begin{eqnarray*}
E^\lozenge_{test}(\lambda)   &=& \sigma^2 + \frac{\alpha^2}{p\tilde{n}}\Tr\Bigg[\mathrm{deg}(\tilde{\Gamma}_{\tilde{n}})\Bigg] + \frac{\alpha^2}{\tilde n}\Tr\Bigg[\mathbb{E}\left[\frac{1}{pn^2}\tilde{H}^{\lozenge\top} H^{\lozenge} Q^{\lozenge }(-\lambda)\mathcal X_n\mathcal X_n^\top Q^{\lozenge }(-\lambda)H^{\lozenge\top} \tilde{H}^{\lozenge}\right]\Bigg]\\
                    & &  + \frac{\sigma^2}{\tilde n}\Tr\Bigg[\mathbb{E}\left[\frac{1}{n^2}\tilde{H}^{\lozenge\top} H^{\lozenge} Q^{\lozenge '}(-\lambda)H^{\lozenge\top} \tilde{H}^{\lozenge} \right]\Bigg]-\frac{2\alpha^2}{\tilde{n}np}\Tr\Bigg[\mathbb{E}\left[ \tilde{\mathcal X}_{\tilde{n}} \mathcal X_n^\top Q^{\lozenge }(-\lambda) H^{\lozenge\top} \tilde{H}^{\lozenge}\right]\Bigg].
\end{eqnarray*}
\end{thm}
{\CB The proof of Theorem \ref{thm:equiv_risk}  is given in Section  \ref{sec:proof_traffic}.}

\subsection{ Square asymptotic equivalents}\label{sec:snd_equiv}

The matrix $H^\lozenge$ defined by \eqref{eq:H-H_lin-diag} being easier to study than $H$, one can now approximate the training and predictive risks with a new expression using only deterministic matrices  to approximate the resolvent of $H^\lozenge$.  These deterministic matrices are derived from the combination of various free probability results. First, we use the linearization trick  \cite{mingo2017free}[Chapter 10] (sometimes called the linear pencil \cite{helton2018applications}), that consists in using the matrix $L$ defined in \eqref{eq:L}. This matrix proves to be useful as, {\CR in Section \ref{sec:proof_square}, we show that $E_{train}^{\lozenge}$ and $E_{test}^{\lozenge}$} can be expressed in terms of  sub-blocks of $\mathfrak{Q}(\Lambda_\lambda)$, the operator-valued resolvent of $L$ defined in \eqref{eq:op_resolvent}, with 
\begin{equation}
\Lambda_\lambda =  -\lambda I_N \otimes E_{11}. \label{eq:Lambda}
\end{equation}

{\CB In the following theorem, it is shown that the resolvent $\mathfrak{Q}(\Lambda)$ of $L$ defined in \eqref{eq:op_resolvent} is well approximated by the solution, denoted $\mathfrak{Q}^\square(\Lambda)$, of the fixed-point equation \eqref{eq:fixedpoint} defined below. The proof of this result is inspired by arguments from \cite{bigotmale}, and we provide its relevant steps in the Appendix.}

\begin{thm}\label{thm:adaptation}
Under Assumptions \ref{hyp:h-analytical} to \ref{hyp:h-bounded}, there exists a unique function $\mathfrak{Q}^{\mathrm{\square}}: \mathrm{D}_N(\mathbb{C})^{+} \rightarrow \mathrm{D}_N(\mathbb{C})^{+}$, analytic in each variable, that solves the following fixed point equation
\begin{equation} \label{eq:fixedpoint}
    \mathfrak{Q}^{\square}(\Lambda)=id_4 \otimes \Delta \left[\left(C_N-\Lambda-\mathcal{R}_N\left(\mathfrak{Q}^{\square}(\Lambda)\right)\right)^{-1}\right],
\end{equation}
for any $\Lambda \in \mathrm{D}_N(\mathbb{C})^{+}$, with
$$
\mathcal{R}_N(\Lambda)=\underset{i=1, \ldots, N}{\operatorname{diag}}\left(\sum_{j=1}^N \frac{{\CB \left( \gamma_{ij}^{(L)}\right)^2}}{N} \Lambda(j, j)\right).
$$
 Let $\left(\gamma_{\max }^{(L)}\right)^2=\max _{i, j} {\CB \left( \gamma_{ij}^{(L)}\right)^2}$,  $0<\delta<1$, and consider $\Lambda \in$ $\mathrm{D}_N(\mathbb{C})^{+}$ satisfying
$$
\Im m \Lambda \geq \left( \frac{\left(\gamma_{\max }^{(L)}\right)^2}{(1-\delta)n} (2\sqrt{2}\left(\gamma_{\max }^{(L)}\right)^3 + \kappa) \right)^{\frac{1}{5}} \mathbb{I}_N.
$$
Then, for any $d>1$, setting
$$
\begin{aligned}
\varepsilon_N(d)= & 6\gamma_{\max }^{(L)} \| \Im m (\Lambda)^{-1}\|^2\sqrt{\frac{2}{N}\log\left( \frac{N+2p}{N^{1-d}}\right)} \\
& +\frac{1}{n}(1+\left(\gamma_{\max }^{(L)}\right)^2\| (\Im m (\Lambda))^{-1} \|^2/\delta ) (2\sqrt{2}\left(\gamma_{\max }^{(L)}\right)^3\| (\Im m (\Lambda))^{-1} \|^4 + \kappa \| (\Im m (\Lambda))^{-1} \|^2),
\end{aligned}
$$
we have, for $N$ large enough,
$$
\mathbb{P}\left(\left\|id_4 \otimes \Delta [\mathfrak{Q}(\Lambda)]-\mathfrak{Q}^{\square}(\Lambda)  \right\| \geq \varepsilon_N(d)\right) \leq 4 N^{1-d}
$$
where $\|\cdot\|$ denotes the operator norm of a matrix.
\end{thm}

In Theorem \ref{thm:adaptation}, the probability that $\left\|id_4 \otimes \Delta [\mathfrak{Q}(\Lambda)]-\mathfrak{Q}^{\square}(\Lambda)\right\|$ is large being bounded by a summable bound for $d$ large enough, one can prove from Borel-Cantelli's lemma that this value tends almost surely to $0$ as $n$ tends to infinity.  In order to obtain the second set of asymptotic equivalents for the  training and predictive risks, we need to approximate $\mathfrak{Q}(\Lambda_\lambda)$ by a deterministic equivalent for $\Lambda_\lambda$ defined by \eqref{eq:Lambda}.
However, directly using  Theorem \ref{thm:adaptation} is not feasible as  $\Lambda_\lambda$ does not satisfy the assumption that ${\CB \Im m (\Lambda_\lambda) > 0}$. Therefore, we approximate $\mathfrak{Q}(\Lambda_\lambda)$ by $\mathfrak{Q}(\Lambda_\lambda + i\eta_N\mathbb{I}_N)$ where $\eta_N$ is quantity that tends to $0$, and we use the deterministic equivalent $\mathfrak{Q}^\square(\Lambda_\lambda + i\eta_N\mathbb{I}_N)$ from Theorem \ref{thm:adaptation}. 
This finally leads to the following corollary
\begin{cor}\label{cor:equiv-ps}
Suppose that Assumptions \ref{hyp:h-analytical} to \ref{hyp:h-bounded} hold, and  denote
\begin{equation}
\eta_N = \left( \frac{\left(\gamma_{\max }^{(L)}\right)^2}{(1-\delta)n} (2\sqrt{2}\left(\gamma_{\max }^{(L)}\right)^3 + \kappa) \right)^{\frac{1}{5}}. \label{eq:etaN}
\end{equation}
For $\mathfrak{Q}^\square$ defined in  Theorem \ref{thm:adaptation} and $\Lambda_\lambda$ defined by \eqref{eq:Lambda}, one has that 
\begin{eqnarray*}
\lim_{n,m,p\rightarrow +\infty}\left\|id_4 \otimes \Delta [\mathfrak{Q}(\Lambda_\lambda + i\eta_N\mathbb{I}_N)]-\mathfrak{Q}^{\square}(\Lambda_\lambda + i\eta_N\mathbb{I}_N)\right\| &=& 0, \mbox{ a.s, }\\ 
\lim_{n,m,p\rightarrow +\infty} \left\|\mathfrak{Q}^{\square}(\Lambda_\lambda + i\eta_N\mathbb{I}_N)-id_4 \otimes \Delta [\mathfrak{Q}(\Lambda_\lambda)]\right\|_{F,n} &=& 0 \mbox{ a.s. }
\end{eqnarray*}
\end{cor}

Finally, we derive the second set of asymptotic equivalents of $E_{train}$ and $E_{test}$ from the sub-blocks of $\mathfrak{Q}^\square$ as follows.

\begin{thm}\label{thm:main}

Suppose that Assumption \ref{hyp:rowsto} holds true as well as Assumptions \ref{hyp:moments-analytical} to \ref{hyp:h-bounded}.
For sake of clarity, we write $ \mathfrak{Q}^\square = \mathfrak{Q}^\square(\Lambda_\lambda+ \mathbf{i}\eta_N I_{N})$ for  $\Lambda_\lambda$ and  $\eta_N$ defined by \eqref{eq:Lambda} and \eqref{eq:etaN} respectively, and we denote by $\mathfrak{Q}^{\square '}$ the derivative of $\mathfrak{Q}^\square(\Lambda_{-z}+ \mathbf{i}\eta_N I_{N})$ with respect to $z$. Then, one has that 
\begin{eqnarray*}
\lim_{\substack{n \to \infty, \; p/n \to c_p\\ m/n \to c_m}} |E_{train} (\lambda) - E^\square_{train}(\lambda)| = 0  & \mathrm{and} & \lim_{\substack{n \to \infty, \; p/n \to c_p\\ m/n \to c_m \; \tilde{n}/n \to \tilde{c}}} |E_{test}(\lambda) - E^\square_{test}(\lambda)| = 0,
\end{eqnarray*}
where the expressions of $E^\square_{train}(\lambda)$ and $E^\square_{test}(\lambda)$ depend on $\theta_{lin}(h)$ in the following way:
\begin{itemize}
\item[-] if $\theta_{lin}(h) = 0$, then
\begin{eqnarray}
E^\square_{train}(\lambda) &=& \lambda^2\alpha^2s^2m_n'(-\lambda)  + \lambda^2\sigma^2 m_n'(-\lambda), \label{eq:Etrain_explicit2} \\
E^\square_{test}(\lambda) &=& \sigma^2 + \alpha^2s^2 + \theta_{chaos}^2(h)(\alpha^2s^2+\sigma^2)(m_n(-\lambda) - \lambda m_n'(-\lambda)) \label{eq:Etest_explicit2},
\end{eqnarray}
\textcolor{black}{where $m_n$ is defined in \eqref{eq:m_n} and denotes the Stieltjes transform  of the Marchenko-Pastur distribution with parameter $\varphi_m = \lim_{n,m \rightarrow +\infty} \frac{n}{m}$ and $m_n'$ denotes its derivative.} 
\item[-] Else, if $ \theta_{lin}(h) \neq 0$,then one has that
\begin{eqnarray*}
E^\square_{train}(\lambda) & = &  \frac{\lambda^2\alpha^2}{\theta_{lin}^2(h)p}\mathrm{Tr}\left[\mathfrak{Q}^{\square '}_{44}\right] + \frac{\lambda^2\sigma^2}{n}\mathrm{Tr}[\mathfrak{Q}^{\square'}_{11}],\\
E^\square_{test}(\lambda)   &=& \sigma^2 + \alpha^2s^2 + \frac{\theta_{chaos}^2(h)\alpha^2}{\theta_{lin}^2(h) p}\Tr[ \mathfrak{Q}^{\square}_{44}] - \frac{\lambda \theta_{chaos}^2(h)\alpha^2}{\theta_{lin}^2(h) p}\Tr[\mathfrak{Q}^{\square '}_{44} ] \\
& +& \frac{\alpha^2 }{p }\Tr\left[ \frac{1}{\tilde n}  \deg ( \tilde{\Gamma}_{\tilde{n}}^\top ) (\mathfrak{Q}^{\square}_{34}+I)(\mathfrak{Q}^{\square}_{43} + I) \right] \\
                    &+&  \frac{\theta_{chaos}^2(h)\sigma^2}{ n}\Tr[\mathfrak{Q}^{\square '}_{22}  ]+ \frac{\theta_{lin}^2(h)\sigma^2 }{ n}\Tr\left[\frac{1}{\tilde n} \deg ( \tilde{\Gamma}_{\tilde{n}}^\top )\mathfrak{Q}^{\square'}_{33}\right]-\frac{2\theta_{lin}(h)\alpha^2 }{p } \Tr\left[\frac{1}{\tilde n}  \deg ( \tilde{\Gamma}_{\tilde{n}}^\top ) (\mathfrak{Q}^{\square}_{43}+I)\right].
\end{eqnarray*}
\end{itemize}
\end{thm}

{\CB Section \ref{sec:proof_free} and} the Appendix  are dedicated to the proof of Theorem \ref{thm:main}. 

\section{Proof of the main results} \label{sec:proof_main}

{\CB This section is devoted to the proofs of Theorem \ref{thm:equiv_risk} (Lozenge asymptotic equivalents) and Theorem \ref{thm:main} (Square asymptotic equivalents).}

\subsection{Derivation of the Lozenge asymptotic equivalents using the traffic theory}\label{sec:proof_traffic}

The expressions {\CB of the training and predictive risks} provided by Lemma \ref{lem:risk_expressions} are not tractable enough because of the non-linearity caused by $h$ in $H$, $\tilde{H}$ and $Q(-\lambda)$. {\CB Fortunately}, Proposition \ref{prop:H-H_lin} allows us to find more convenient approximations for $E_{train}$ and $E_{test}$ by replacing $H$, $X_n$, and $Q$ with $H^\lozenge$, $\mathcal X_n$, and $Q^{\lozenge}(\lambda)$, as precisely stated in Theorem \ref{thm:equiv_risk} that we prove below.

\textcolor{black}{\begin{lem}\label{lem:Q-analytic}
Let $M$ be a random real symmetric positive definite matrix of size $n\times n$, and define
$$
Q(z)=(M-zI_n)^{-1}, \qquad z\in \mathbb{C}\setminus \mathbb{R}_+.
$$
Then, the matrix-valued map
    $
    z \longmapsto \mathbb{E}[Q(z)]
    $
    is analytic on $\mathbb{C}\setminus \mathbb{R}_+$.
\end{lem}
\begin{proof}
Since $M$ is real symmetric positive definite, all its eigenvalues are real and strictly positive. Therefore, for almost every realization $\omega$,
$$
\operatorname{Sp}(M(\omega)) \subset (0,\infty)=\mathbb{R}_+.
$$
Hence for every $z\in \mathbb{C}\setminus \mathbb{R}_+$, the matrix $M(\omega)-zI_n$ is invertible, so $Q(z,\omega) = (M(\omega)-zI_n)^{-1}$ is well defined almost surely.
\medskip
Fix a realization $\omega$, since $M(\omega)$ is symmetric, there exist an orthogonal matrix $U(\omega)$ and eigenvalues $\lambda_1(\omega),\dots,\lambda_n(\omega)>0$ such that
$$
M(\omega)=U(\omega)\,\mathrm{diag}(\lambda_1(\omega),\dots,\lambda_n(\omega))\,U(\omega)^\top.
$$
Therefore
$$
Q(z,\omega)
=
U(\omega)\,
\mathrm{diag}\!\left(
\frac1{\lambda_1(\omega)-z},\dots,\frac1{\lambda_n(\omega)-z}
\right)
U(\omega)^\top.
$$
For every $i,j$,
$$
Q_{ij}(z,\omega)
=
\sum_{k=1}^n U_{ik}(\omega)U_{jk}(\omega)\frac1{\lambda_k(\omega)-z}.
$$
This is a finite sum of holomorphic functions of $z$ on $\mathbb{C}\setminus \mathbb{R}_+$, hence
$$
z \longmapsto Q_{ij}(z,\omega)
$$
is holomorphic on $\mathbb{C}\setminus \mathbb{R}_+$.
\medskip
Let $K\subset \mathbb{C}\setminus \mathbb{R}_+$ be compact, and define
$
\delta_K := \inf_{z\in K} \operatorname{dist}(z,\mathbb{R}_+) > 0.
$
For every realization $\omega$, since $\operatorname{Sp}(M(\omega))\subset \mathbb{R}_+$, we have
$$
\operatorname{dist}(z,\operatorname{Sp}(M(\omega)))
\ge \operatorname{dist}(z,\mathbb{R}_+)
\ge \delta_K.
$$
Using the standard resolvent bound for symmetric matrices,
$$
\|(M(\omega)-zI_n)^{-1}\|
\le
\frac1{\operatorname{dist}(z,\operatorname{Sp}(M(\omega)))},
$$
we obtain
$
\|Q(z,\omega)\| \le \frac1{\delta_K}, z\in K.
$
Hence, for all $i,j$,
$$
|Q_{ij}(z,\omega)| \le \|Q(z,\omega)\| \le \frac1{\delta_K},
\qquad z\in K.
$$
Thus, for every $z\in K$,
$$
\mathbb{E}[|Q_{ij}(z)|] \le \frac1{\delta_K} < \infty,
$$
and the quantity $\mathbb{E}[Q_{ij}(z)]$ is well defined.
\medskip
Fix $z_0\in \mathbb{C}\setminus \mathbb{R}_+$ and indices $i,j$. Set
$
\delta := \frac12 \operatorname{dist}(z_0,\mathbb{R}_+) > 0.
$
Then for all $h\in \mathbb{C}$ such that $|h|<\delta$, one has that
$
\operatorname{dist}(z_0+h,\mathbb{R}_+) \ge \delta,
$
so
$$
\|Q(z_0+h,\omega)\| \le \frac1\delta,
\qquad
\|Q(z_0,\omega)\| \le \frac1\delta.
$$
Now we use the resolvent identity:
$$
Q(z_0+h,\omega)-Q(z_0,\omega)
=
h\,Q(z_0+h,\omega)Q(z_0,\omega).
$$
Therefore,
$$
\frac{Q_{ij}(z_0+h,\omega)-Q_{ij}(z_0,\omega)}{h}
=
\bigl(Q(z_0+h,\omega)Q(z_0,\omega)\bigr)_{ij}.
$$
Hence
$$
\left|
\frac{Q_{ij}(z_0+h,\omega)-Q_{ij}(z_0,\omega)}{h}
\right|
\le
\|Q(z_0+h,\omega)\|\,\|Q(z_0,\omega)\|
\le
\frac1{\delta^2}.
$$
This deterministic bound is integrable.\\
Since $z\mapsto Q_{ij}(z,\omega)$ is holomorphic for almost every $\omega$, we have
$$
\frac{Q_{ij}(z_0+h,\omega)-Q_{ij}(z_0,\omega)}{h}
\longrightarrow
\frac{\partial}{\partial z}Q_{ij}(z_0,\omega)
\qquad\text{as } h\to 0.
$$
By the dominated convergence theorem,
$$
\frac{\mathbb{E}[Q_{ij}(z_0+h)]-\mathbb{E}[Q_{ij}(z_0)]}{h}
=
\mathbb{E}\!\left[
\frac{Q_{ij}(z_0+h)-Q_{ij}(z_0)}{h}
\right]
\longrightarrow
\mathbb{E}\!\left[
\frac{\partial}{\partial z}Q_{ij}(z_0)
\right].
$$
Thus $z\mapsto \mathbb{E}[Q_{ij}(z)]$ is complex differentiable at $z_0$. Since $z_0$ is arbitrary, it is holomorphic on $\mathbb{C}\setminus \mathbb{R}_+$. This proves the result.
\end{proof}}

\begin{proof}[Proof of Theorem \ref{thm:equiv_risk}]
{\color{black}
Let $C_H$ be an almost-sure bound given by Assumption \ref{hyp:h-bounded} such that $\EE[\|H^\top H/n\|^k]\leq C_H^k$ for $k\in \mathbb{N^*}$. For $H^{\lozenge}$, the diagonal matrices $D_{lin}(h)$ and $D_{chaos}(h)$ are uniformly bounded under Assumptions \ref{hyp:h-analytical} and \ref{hyp:bounded_profile}, while the Gaussian matrices $\mathcal W$, $\mathcal X_n$, and $Z^G$ have bounded spectral norms after normalization. Hence there exists an almost-sure constant $C_{\lozenge}$ such that $\EE[\|H^{\lozenge\top}H^{\lozenge}/n\|^k]\le C_{\lozenge}^k$ for all sufficiently large $n$ (See \cite{vershynin2020high}[Theorem 4.4.3]). Choosing $z\in\mathbb C\setminus\mathbb R^+$ with $|z|>2\max(C_H,C_{\lozenge})$, the Neumann series
$$
Q(z)=\frac{1}{z}\sum_{k\ge0}\Big(\frac{H^\top H}{zn}\Big)^k,
\qquad
Q^{\lozenge}(z)=\frac{1}{z}\sum_{k\ge0}\Big(\frac{H^{\lozenge\top}H^{\lozenge}}{zn}\Big)^k
$$
converge absolutely and can be differentiated term by term.
By differentiating the series term by term, one has
\begin{eqnarray*}
\frac{1}{n}\mathrm{Tr}\Bigg[\mathbb{E}\left[Q'(z)\frac{X_n X_n^\top}{p} -Q^{\lozenge '}(z)\frac{\mathcal X_n \mathcal X_n^\top}{p} \right]\Bigg] &=& \frac{1}{z^2}\sum_{k\geq 0}\frac{1}{n}\mathrm{Tr}\Bigg[\mathbb{E}\left[ \left(\frac{H^{\lozenge\top} H^\lozenge }{zn}\right)^k\frac{\mathcal X_n \mathcal X_n^\top}{p} -  \left(\frac{H^{\top} H }{zn}\right)^k\frac{ X_n X_n^\top}{p}\right]\Bigg]\\
&& + \frac{1}{z^2}\sum_{k\geq 1}\frac{k}{n}\mathrm{Tr}\Bigg[\mathbb{E}\left[ \left(\frac{H^{\lozenge\top} H^\lozenge }{zn}\right)^k\frac{\mathcal X_n \mathcal X_n^\top}{p} -  \left(\frac{H^{\top} H }{zn}\right)^k\frac{ X_n X_n^\top}{p}\right]\Bigg].
\end{eqnarray*}
Each terms of these sums tends to zero as $n$ goes to $+\infty$ as stated by \eqref{thm:H-H_lin+tilde}. Moreover, $A,B \mapsto \mathbb{E}[\frac{1}{n}\Tr[AB^\ast]]$ being an Hermitian product, one has from the Cauchy-Schwarz and triangle inequalities that
\begin{eqnarray*}
\left|\mathbb{E}\left[\frac{1}{n}\mathrm{Tr}\Bigg[ \left(\frac{H^{\lozenge\top} H^\lozenge }{zn}\right)^k\frac{\mathcal X_n \mathcal X_n^\top}{p} -  \left(\frac{H^{\top} H }{zn}\right)^k\frac{ X_n X_n^\top}{p}\Bigg]\right]\right| &\leq& \sqrt{\mathbb{E}\left[\left\|\left(\frac{H^\top H}{zn}\right)^k \right\|_{F,n}^2\right]} \sqrt{\mathbb{E}\left[\left\|\frac{ X_nX_n^\top }{\sqrt{p}}\right\|_{F,n}^2\right]} \\
&& + \sqrt{\mathbb{E}\left[\left\|\left(\frac{H^{\lozenge\top} H^\lozenge}{zn}\right)^k \right\|_{F,n}^2\right]} \sqrt{\mathbb{E}\left[\left\|\frac{ \mathcal X_n\mathcal X_n^\top }{\sqrt{p}}\right\|_{F,n}^2\right]},\\
&\leq& \sqrt{\mathbb{E}\left[\left\|\frac{H^\top H}{zn} \right\|^{2k}\right]} \sqrt{\mathbb{E}\left[\left\|\frac{ X_nX_n^\top }{\sqrt{p}}\right\|_{F,n}^2\right]} \\
&& + \sqrt{\mathbb{E}\left[\left\|\frac{H^{\lozenge\top} H^\lozenge}{zn}\right\|^{2k}\right]} \sqrt{\mathbb{E}\left[\left\|\frac{ \mathcal X_n\mathcal X_n^\top }{\sqrt{p}}\right\|_{F,n}^2\right]}.
\end{eqnarray*}
Since $|z| > 2 \ \underset{n\geq n_0}{\mathrm{sup}}\lbrace\parallel \frac{H^\top H}{n} \parallel ; \parallel \frac{H^{\lozenge\top} H^{\lozenge}}{n} \parallel\rbrace  $, one has that  $\sqrt{\mathbb{E}\left[\left\|\left(\frac{H^\top H}{zn}\right)^k \right\|^2\right]},\sqrt{\mathbb{E}\left[\left\|\left(\frac{H^{\lozenge\top} H^\lozenge}{zn}\right)^k \right\|^2\right]} <\frac{1}{2^k}$, then 
$$
\left|\mathbb{E}\left[\frac{1}{n}\mathrm{Tr}\Bigg[ \left(\frac{H^{\lozenge\top} H^\lozenge }{zn}\right)^k\frac{\mathcal X_n \mathcal X_n^\top}{p} -  \left(\frac{H^{\top} H }{zn}\right)^k\frac{ X_n X_n^\top}{p}\Bigg]\right]\right| < \left(\frac{1}{2}\right)^k \left[ \sqrt{\mathbb{E}\left[\left\|\frac{  X_nX_n^\top }{\sqrt{p}}\right\|_{F,n}^2\right]} + \sqrt{\mathbb{E}\left[\left\|\frac{ \mathcal X_n\mathcal X_n^\top }{\sqrt{p}}\right\|_{F,n}^2\right]}\right].
$$
 The terms $\mathbb{E}\left[\left\|\frac{  X_n X_n^\top }{\sqrt{p}}\right\|_{F,n}^2\right]$ and $\mathbb{E}\left[\left\|\frac{ \mathcal X_n\mathcal X_n^\top }{\sqrt{p}}\right\|_{F,n}^2\right]$ are actually the second moment of $\frac{  X_n X_n^\top }{\sqrt{p}}$ and $\frac{ \mathcal X_n\mathcal X_n^\top }{\sqrt{p}}$. The matrices $X'_n$ and $X^G$ (See Equations \eqref{eq:WX_prime} and \eqref{eq:WX_G}) being random matrices whose iid entries  are centered with finite moments of all order, the second moment $\mathbb{E}\left[\left\|\frac{  X'_n X_n^{'\top} }{\sqrt{p}}\right\|_{F,n}^2\right]$ and $\mathbb{E}\left[\left\|\frac{ \mathcal X^G\mathcal X^{G\top} }{\sqrt{p}}\right\|_{F,n}^2\right]$ can be proved to converge as $n\rightarrow +\infty$ from the moment method. One can deduce from this that $\mathbb{E}\left[\left\|\frac{  X_n X_n^\top }{\sqrt{p}}\right\|_{F,n}^2\right]$ and $\mathbb{E}\left[\left\|\frac{ \mathcal X_n\mathcal X_n^\top }{\sqrt{p}}\right\|_{F,n}^2\right]$ are also bounded since it is proved in \cite{DaboMale} that the presence of variance profiles does not affect the  proof of convergence of these quantities using the moment method. Then 
\begin{eqnarray}
    &&\sum_{k\geq 0}\left(\frac{1}{2}\right)^k \left[ \sqrt{\mathbb{E}\left[\left\|\frac{  X_nX_n^\top }{\sqrt{p}}\right\|_{F,n}^2\right]} + \sqrt{\mathbb{E}\left[\left\|\frac{ \mathcal X_n\mathcal X_n^\top }{\sqrt{p}}\right\|_{F,n}^2\right]}\right]\\
    &&+\sum_{k\geq 1}k\left(\frac{1}{2}\right)^k \left[ \sqrt{\mathbb{E}\left[\left\|\frac{  X_nX_n^\top }{\sqrt{p}}\right\|_{F,n}^2\right]} + \sqrt{\mathbb{E}\left[\left\|\frac{ \mathcal X_n\mathcal X_n^\top }{\sqrt{p}}\right\|_{F,n}^2\right]}\right] < +\infty.\nonumber
\end{eqnarray}
Thus, we obtain from the dominated convergence theorem, that  
\begin{eqnarray}\label{eq:Q-Q_lin}
\frac{1}{n}\mathrm{Tr}\Bigg[\mathbb{E}\left[Q'(z)\frac{X_n X_n^\top}{p} -Q^{\lozenge '}(z)\frac{\mathcal X_n \mathcal X_n^\top}{p} \right]\Bigg] \xrightarrow[n\rightarrow +\infty]{} 0
\end{eqnarray}
We now extend this convergence to all $z\in\mathbb C\setminus\mathbb R^+$. Define
$$
f_n(z)=\frac{1}{n}\Tr\Bigg[\mathbb E\Big[Q'(z)\frac{X_nX_n^\top}{p}-Q^{\lozenge\,'}(z)\frac{\mathcal X_n\mathcal X_n^\top}{p}\Big]\Bigg].
$$
Each function $f_n$ is analytic on $\mathbb C\setminus\mathbb R^+$ because the resolvents $Q(z)$ and $Q^{\lozenge}(z)$ are analytic there, and expectation preserves analyticity on compact subsets by dominated convergence, as stated in Lemma \ref{lem:Q-analytic}. In addition, the resolvent estimate
$$
\|Q(z)\|\le \frac{1}{\mathrm{dist}(z,\mathbb R^+)},
\qquad
\|Q^{\lozenge}(z)\|\le \frac{1}{\mathrm{dist}(z,\mathbb R^+)},
$$
shows that $(f_n)_n$ is locally uniformly bounded on $\mathbb C\setminus\mathbb R^+$. By Montel's theorem, every subsequence has a further subsequence converging locally uniformly to an analytic limit. Since the limit vanishes on the nonempty open set $\{z\in\mathbb C\setminus\mathbb R^+: |z|>2\max(C_H,C_{\lozenge})\}$ by \eqref{eq:Q-Q_lin}, the identity theorem implies that this limit is identically zero on $\mathbb C\setminus\mathbb R^+$. Consequently $f_n\to0$ locally uniformly on $\mathbb C\setminus\mathbb R^+$.
Applying the same argument to the remaining term in the expression of $E_{train}(\lambda)$ given by Lemma \ref{lem:risk_expressions}, we obtain
$$
E_{train}(\lambda)=\frac{\lambda^2\alpha^2}{n}\Tr\Bigg[\mathbb E\Big[Q^{\lozenge\,'}(-\lambda)\frac{\mathcal X_n\mathcal X_n^\top}{p}\Big]\Bigg]+
\frac{\lambda^2\sigma^2}{n}\Tr\Bigg[\mathbb E\big[Q^{\lozenge\,'}(-\lambda)\big]\Bigg]+o(1).
$$
The proof for $E_{test}(\lambda)$ is identical in spirit: each term in Lemma \ref{lem:risk_expressions} is a trace of an expectation of matrix polynomials in $(H,W,X_n,\tilde H,\tilde X_{\tilde n},Q(-\lambda))$, and Proposition \ref{prop:H-H_lin} applies to every such polynomial after replacing $(H,W,X_n,\tilde H,\tilde X_{\tilde n})$ with $(H^{\lozenge},\mathcal W,\mathcal X_n,\tilde H^{\lozenge},\tilde{\mathcal X}_{\tilde n})$. This proves that $E_{test}(\lambda)-E_{test}^{\lozenge}(\lambda)\to0$ as well.}
\end{proof}

\subsection{Derivation of the Square asymptotic equivalents using free probability} \label{sec:proof_free}

In order to prove Theorem \ref{thm:main}, we now construct a deterministic equivalent of $Q^\lozenge(\lambda)$, since $E^\lozenge_{train}$ and $E^\lozenge_{test}$ depend on this matrix. The linearization trick \cite{mingo2017free} is often used to approximate resolvent matrices such as $Q^\lozenge(\lambda)$. In our case, it turns out that this method \rev{allows us to approximate} more complicated expressions that are matrix polynomials composing {\CB the expressions of the training and predictive risks that involve} $Q^\lozenge (\lambda), H^\lozenge, \tilde X, \tilde W$. 
 The linearization trick, that we apply in this paper, consists in using the matrix $L$ defined by \eqref{eq:L} which is  a matrix polynomial of degree one depending on $W,X$ and $Z$. 
{\CB Following  \cite{mingo2017free}, the choice of $L$ is motivated by  Definition \ref{def:linearization} of the linearization of a matrix that is recalled in the Appendix}. This trick is a natural way to ease computations whenever confronted to rational fraction of matrices. Moreover, a systematic way of obtaining a linearization of a polynomial matrix is described in \cite{mingo2017free} and the case of rational fraction as been studied in \cite{10.22034/aot.1702-1126}. Now, we recall that, when assuming that $W$ has a constant variance profile, the ``linear-plus-chaos" approximation writes as   Equation \eqref{eq:H-H_lin-diag}. 
This polynomial expression for $H^{\lozenge}$ {\CB has guided} us to propose the matrix $L$ given by \eqref{eq:L} as its  linearization, which is not unique but relevant for our purposes.

Let $\Lambda \in \mathbb{C}^{N}$ be a diagonal matrix such that $L - \Lambda$ is invertible (the condition $\Im m(\Lambda) > 0$ ensures the invertibility, see Lemma \ref{lem:Im_pos}). Recall that $\mathfrak{Q}(\Lambda) = (L - \Lambda)^{-1}$ is the resolvent matrix of $L$. We need to compute this resolvent for $\Lambda_\lambda = -\lambda I_n \otimes E_{11}$, and the resulting expression of $\mathfrak{Q}(\Lambda_\lambda)$ is obtained using the Schur complement formula: 
\begin{equation} \label{eq:mathfrakQ}
\mathfrak{Q} (\Lambda_\lambda) = \begin{pmatrix} Q^{\lozenge }(-\lambda) & \frac{Q^{\lozenge }(-\lambda)H^{\lozenge\top}}{\sqrt{n}} & - \frac{Q^{\lozenge }(-\lambda)H^{\lozenge\top} \mathcal W}{\sqrt{np}} &   -\frac{Q^{\lozenge }(-\lambda)D_{lin}(h)\mathcal X_n }{\sqrt{n}}\\
\frac{H^{\lozenge}Q^{\lozenge }(-\lambda)}{\sqrt{n}} &  \frac{H^{\lozenge}Q^{\lozenge }(-\lambda)H^{\lozenge\top}}{n} - I_m  & \frac{\mathcal W}{\sqrt{p}} -  \frac{H^{\lozenge}Q^{\lozenge }(-\lambda)H^{\lozenge\top} \mathcal W}{n\sqrt{p}} & -\frac{H^{\lozenge}Q^{\lozenge }(-\lambda)\mathcal X_n D_{lin}(h)}{n} \\
-\frac{\mathcal W^\top H^{\lozenge} Q^{\lozenge }(-\lambda)}{\sqrt{np}} & \frac{\mathcal W^\top}{\sqrt{p}} + \frac{\mathcal W^\top H^{\lozenge}Q^{\lozenge }(-\lambda)H^{\lozenge\top} }{n\sqrt{p}}  & \frac{\mathcal W^\top H^{\lozenge}Q^{\lozenge }(-\lambda)H^{\lozenge\top} \mathcal W}{np} - \frac{\mathcal W^\top \mathcal W}{p} & \frac{\mathcal W^\top H^{\lozenge}Q^{\lozenge }(-\lambda)D_{lin}(h)\mathcal X_n}{n\sqrt{p}} - I_p \\
-\frac{\mathcal X_n^\top D_{lin}(h)Q^{\lozenge }(-\lambda) }{\sqrt{n}} & -\frac{\mathcal X_n^\top D_{lin}(h)Q^{\lozenge }(-\lambda) H^{\lozenge\top}}{n} & \frac{\mathcal X_n^\top D_{lin}(h)Q^{\lozenge }(-\lambda) H^{\lozenge\top} \mathcal W}{n\sqrt{p}} - I_p  & \frac{\mathcal X_n^\top D_{lin}(h)Q^{\lozenge }(-\lambda)D_{lin}(h)\mathcal X_n }{n}
\end{pmatrix}
\end{equation}
{\CR As shown in Section \ref{sec:proof_square},} the first set of asymptotic equivalents $E_{train}^\lozenge$ and $E_{test}^\lozenge$ can be expressed in terms of some sub-blocks of $\mathfrak{Q} (\Lambda_\lambda)$. Therefore, it is now relevant to find a deterministic equivalent of $\mathfrak{Q} (\Lambda_\lambda)$ to obtain the second set of asymptotic equivalents for the training and testing errors. This deterministic equivalent of $\mathfrak{Q}(\Lambda_\lambda)$ is obtained by following the methodology developed in \cite{bigotmale}. This work studies deformed GUE matrices with a variance profile defined as
$M_N = \Upsilon \circ A_N + B_N \in \mathbb{R}^{N\times N}$, where $A_N$ is a GUE matrix  (that is $A_N$ is {\CB  Hermitian  such that its upper diagonal coefficients $(a_{ij})_{i< j}$ form a sequence of iid standard complex Gaussian variables and its diagonal elements $(a_{ii})$ form a sequence of iid standard real Gaussian variables}), $\Upsilon = (\gamma_N(i,j))$ and $B_N$ are deterministic Hermitian matrices and {\CB $\Gamma = (  \gamma_N^2(i, j))$ is a variance profile matrix} with  non-negative  real entries. Under this framework,  a deterministic equivalent of the diagonal of the resolvent {\CB of  $M_N$ is provided in \cite{bigotmale}  in the sense of Theorem \ref{thm:bigotmale}  that is recalled in the Appendix for completeness}.

Now, as explained in Section \ref{sec:notations}, the linearization $L$ {\CB of  $H^{\lozenge}$} can be interpreted as a deformed variance profiled GOE matrix, see the  key relation \eqref{eq:key}.
It would be convenient to use Theorem \ref{thm:bigotmale} on {\CB the decomposition \eqref{eq:key} of the linearization $L$} to obtain a relevant deterministic approximation of its resolvent $\mathfrak{Q}(\Lambda_\lambda)$. However, we cannot use Theorem \ref{thm:bigotmale} since {\CB $A_N$ in \eqref{eq:key}} is not a GUE matrix. Moreover, Theorem \ref{thm:bigotmale} provides a deterministic equivalent of the diagonal of the resolvent whereas we need a deterministic equivalent for the diagonal of the blocks of the resolvent $\mathfrak{Q}(\Lambda_\lambda)$ via the  map $id\otimes \Delta$. Nevertheless,  by changing some of the proofs in \cite{bigotmale}, we have been able to adapt this result to our setting leading to Theorem \ref{thm:adaptation}. These changes are highlighted in the Appendix.

\subsection{Proof of Theorem \ref{thm:main}}\label{sec:proof_square}
In this section, we suppose that the assumptions of Theorem \ref{thm:main} hold. Under Assumption \ref{hyp:rowsto}, the ``linear-plus-chaos" approximation simplifies to Equation \eqref{eq:H-H_lin-constant}. Note that the expression of $E_{train}^\lozenge$ does not change under this assumption,  whereas a new expression for $E_{test}^\lozenge$ arises that is given below:
\begin{eqnarray}
E_{test}^\lozenge(\lambda)&=& \sigma^2 + \frac{\alpha^2}{p}\Tr[\mathbb{E}[\frac{1}{\tilde n}\tilde{\mathcal X}_{\tilde{n}} \tilde{\mathcal X}_{\tilde{n}}^\top]] + \frac{\theta_{chaos}^2(h)\alpha^2}{ n}\Tr[\mathbb{E}[\frac{1}{p} (Q^{\lozenge }(-\lambda) - \lambda Q^{\lozenge '}(-\lambda)) \mathcal X_n\mathcal X_n^\top ]] \label{eq:Etest} \\
                    &+& \frac{\theta_{lin}^2(h)\alpha^2}{p}\Tr[\mathbb{E}[\frac{1}{\tilde n}\tilde{\mathcal X}_{\tilde{n}}^\top \tilde{\mathcal X}_{\tilde{n}}]\mathbb{E}[\frac{1}{pn^2} \mathcal W^\top H^{\lozenge} Q^{\lozenge }(-\lambda)\mathcal X_n\mathcal X_n^\top Q^{\lozenge }(-\lambda)H^{\lozenge\top} \mathcal W ]]  \nonumber \\
                    & &  + \frac{\theta_{chaos}^2(h)\sigma^2}{ n}\Tr[\mathbb{E}[Q^{\lozenge  }(-\lambda) - \lambda Q^{\lozenge '}(-\lambda)  ]]\nonumber\\
                    &&+ \frac{\theta_{lin}^2(h)\sigma^2}{n}\Tr[\mathbb{E}[\frac{1}{\tilde n}\tilde{\mathcal X}_{\tilde{n}}^\top \tilde{\mathcal X}_{\tilde{n}}]\mathbb{E}[\frac{1}{pn}\mathcal W^\top H^{\lozenge} Q^{\lozenge '}(-\lambda)H^{\lozenge\top} \mathcal W ]] \nonumber \\ &-&\frac{2\theta_{lin}^2(h)\alpha^2}{np}\Tr[\mathbb{E}[\frac{1}{\tilde n}\tilde{\mathcal X}_{\tilde{n}}^\top \tilde{\mathcal X}_{\tilde{n}}]\mathbb{E}[  \mathcal X_n^\top Q^{\lozenge } H^{\lozenge\top} \mathcal W/\sqrt{p}]]. \nonumber
\end{eqnarray}

\subsubsection{Second set of asymptotic equivalents when $\theta_{lin}(h) = 0$}
Let us first focus on the case where $\theta_{lin}(h) = 0$. In this case, $H^\lozenge = \theta_{chaos}(h)Z$ and
{\CB the expressions of $E_{train}^\lozenge(\lambda)$ and $E_{test}^\lozenge(\lambda)$ simplify as follows:}
\begin{eqnarray*}
E_{train}^\lozenge(\lambda) &=& \frac{\lambda^2\alpha^2 s^2}{n}\mathrm{Tr}\Bigg[\mathbb{E}\left[Q^{\lozenge '}(-\lambda)\right] \Bigg] + \frac{\lambda^2\sigma^2}{n}\mathrm{Tr}[\mathbb{E}[Q^{\lozenge '}(-\lambda)]],\\
E_{test}^\lozenge(\lambda)  &=& \sigma^2 + \alpha^2 s^2 + \frac{\theta_{chaos}^2(h)(\alpha^2s^2 + \sigma^2)}{ n}\Tr[\mathbb{E}[ Q^{\lozenge }(-\lambda) - \lambda Q^{\lozenge '}(-\lambda)] ],
\end{eqnarray*}
{\CB where we use the fact that $Q^\lozenge$ becomes independent from $\mathcal X_n$ and $\mathcal W$} and the property that
$\frac{1}{p}\mathbb{E}[\frac{1}{n} \mathcal X_{n} \mathcal X_{n}^\top] = \frac{s^2}{n} I_n $ and $\frac{1}{p} \ \mathbb{E}[\frac{1}{\tilde n}\tilde{\mathcal X}_{\tilde{n}} \tilde{\mathcal X}_{\tilde{n}}^\top] = \frac{s^2}{\tilde{n}} I_{\tilde{n}}$ by Assumption \ref{hyp:rowsto}. 
 Since $H^{\lozenge} = \theta_{chaos}(h)Z$ where $Z$ is a random matrix with iid entries sampled from the standard Gaussian distribution, it is well-known  \cite{MR2567175} that $$\lim_{n \to \infty, \; n/m \to \varphi_m} \frac{1}{n} \mathrm{Tr}[Q^{\lozenge }(-\lambda)] = m_n(-\lambda) \quad \mbox{and} \quad \lim_{n \to \infty, \; n/m \to \varphi_m} \frac{1}{n} \mathrm{Tr}[Q^{\lozenge '}(-\lambda)] = m'_n(-\lambda) , \; \mbox{almost surely},$$   where $m_n(-\lambda)$ is the Stieltjes transform \eqref{eq:m_n} of the Marchenko-Pastur distribution, which yields the expressions \eqref{eq:Etrain_explicit2} and \eqref{eq:Etest_explicit2} of the training and predictive risks.

\subsubsection{Second set of asymptotic equivalents when $\theta_{lin}(h) \neq 0$}
{\CB Suppose now that $\theta_{lin}(h) \neq 0$. In this case, {\CR using the expression \eqref{eq:mathfrakQ} of $\mathfrak{Q} (\Lambda_\lambda)$ and Equation  \eqref{eq:Etest}, it follows that} $E_{train}^{\lozenge}$ and $E_{test}^{\lozenge}$ can be expressed in terms of  sub-blocks of $\mathfrak{Q} = \mathfrak{Q} (\Lambda_\lambda)$ as follows:}  
\begin{eqnarray*}
E_{train}^{\lozenge}(\lambda) &=& \frac{\lambda^2\alpha^2}{\theta_{lin}^2(h)p}\mathrm{Tr}\Bigg[\mathbb{E}[\mathfrak{Q}_{44}']\Bigg] + \frac{\lambda^2\sigma^2}{n}\mathrm{Tr}[\mathbb{E}[\mathfrak{Q}_{11}']], \\
E_{test}^{\lozenge}(\lambda)   &=& \sigma^2 + \frac{\alpha^2}{p}\Tr[\mathbb{E}[\frac{1}{\tilde n}\tilde{\mathcal X}_{\tilde{n}} \tilde{\mathcal X}_{\tilde{n}}^\top]] + \frac{\theta_{chaos}^2(h)\alpha^2}{\theta_{lin}^2(h) p}\Tr[\mathbb{E}[ \mathfrak{Q}_{44}]] - \frac{\lambda \theta_{chaos}^2(h)\alpha^2}{\theta_{lin}^2(h) p}\Tr[\mathbb{E}[ \mathfrak{Q}_{44}' ]] \\
                    &+& \frac{\alpha^2}{p}\Tr[\mathbb{E}[\frac{1}{\tilde n}\tilde{\mathcal X}_{\tilde{n}}^\top \tilde{\mathcal X}_{\tilde{n}}]\mathbb{E}[ (\mathfrak{Q}_{34}+I)(\mathfrak{Q}_{43} + I) ]]  + \frac{\theta_{chaos}^2(h)\sigma^2}{ n}\Tr[\mathbb{E}[\mathfrak{Q}_{22}'  ]]\\
                    &+& \frac{\theta_{lin}^2(h)\sigma^2}{ n}\Tr[\mathbb{E}[\frac{1}{\tilde n}\tilde{\mathcal X}_{\tilde{n}}^\top \tilde{\mathcal X}_{\tilde{n}}]\mathbb{E}[\mathfrak{Q}'_{33}]]
                    -\frac{2\theta_{lin}(h)\alpha^2}{p}\Tr[\mathbb{E}[\frac{1}{\tilde n}\tilde{\mathcal X}_{\tilde{n}}^\top \tilde{\mathcal X}_{\tilde{n}}]\mathbb{E}[ (\mathfrak{Q}_{43}+I)]].
\end{eqnarray*}
First, the convergence of the two terms appearing in the decomposition of $E_{train}^{\lozenge}(\lambda)$ is an immediate consequence of  the following corollary, which \rev{allows us to prove} that $\lim_{n \to + \infty} |E_{train}^{\lozenge}(\lambda) - E_{train}^{\square}(\lambda)| = 0$ where
$$
E_{train}^{\square}(\lambda) = \frac{\lambda^2\alpha^2}{\theta_{lin}^2(h)p}\mathrm{Tr}\left[\mathfrak{Q}^{\square '}_{44}\right] + \frac{\lambda^2\sigma^2}{n}\mathrm{Tr}[\mathfrak{Q}^{\square'}_{11}].
$$
\begin{cor}\label{cor:deriv}
Let $z\in \mathbb{C}\setminus \mathbbm{R}^+$ , we denote by $\mathfrak{Q}'(\Lambda_{-z})$, $\mathfrak{Q}'(\Lambda_{-z}+i\eta_N I_{N})$ and $\mathfrak{Q}^{\square '}(\Lambda_{-z}+i\eta_N I_{N})$ the respective derivatives of $\mathfrak{Q}(\Lambda_{-z})$, $\mathfrak{Q}(\Lambda_{-z}+i\eta_N I_{N})$ and $\mathfrak{Q}^\square(\Lambda_{-z}+i\eta_n I_{N})$ with respect to $z$, where $\eta_N$ is defined in \eqref{eq:etaN}. Then the following limits hold almost surely
\begin{eqnarray*}
\left\|\mathfrak{Q}'(\Lambda_{-z}) - \mathfrak{Q}'(\Lambda_{-z}+i\eta_N I_{N})\right\|_{F,n} \xrightarrow{} 0 \quad \mbox{ and } \quad \left\|\mathfrak{Q}^{\square '}(\Lambda_{-z}+i\eta_n I_{N}) - id_4 \otimes \Delta [\mathfrak{Q}'(\Lambda_{-z}+i\eta_n I_{N})]\right\|_{F,n} \xrightarrow{} 0.
\end{eqnarray*}
\end{cor}

\begin{proof}[Proof of Corollary \ref{cor:deriv}]
Let $z\in \mathbb{C}\setminus \mathbbm{R}^+$ such that $\Im m(z) > 0$ and denote $\tilde\Lambda_{z} = \Lambda_{-z}+i\eta_N I_{N}$. We first need to notice that $\mathfrak{Q}'(\Lambda_{-z})$, $\mathfrak{Q}'(\Lambda_{-z}+i\eta_N I_{N})$ and $\mathfrak{Q}^{\square '}(\Lambda_{-z}+i\eta_n I_{N})$ are analytic functions, see \cite{bigotmale}[Lemma 6.4 and Theorem 1.1], then the Cauchy integral formula yields that 
    \begin{eqnarray*}
\mathfrak{Q}'(\Lambda_{-z}) &=& \frac{1}{2\pi i} \int_\rho \frac{\mathfrak{Q}(\Lambda_{w})}{(w-z)^2} dw, \quad \mathfrak{Q}'(\Lambda_{-z}+i\eta_N I_{N}) = \frac{1}{2\pi i} \int_\rho \frac{\mathfrak{Q}(\Lambda_{w}+i\eta_N I_{N})}{(w-z)^2} dw \\  \mathfrak{Q}^{\square'}(\Lambda_{-z}+i\eta_n I_{N}) &=& \frac{1}{2\pi i} \int_\rho \frac{\mathfrak{Q}^{\square'}(\Lambda_w+i\eta_n I_{N})}{(w-z)^2} dw,
\end{eqnarray*}
where $\rho$ is a path around $z$ in $\lbrace w\in \mathbb C \setminus \mathbb R^+ | \Im m(w) > 0 \rbrace$.

Hence, one has from the triangle inequality that 
$
\left\|\mathfrak{Q}'(\Lambda_{-z}) - \mathfrak{Q}'(\tilde\Lambda_{z})\right\|_{F,n} \leq \frac{1}{2\pi i} \int_\rho \frac{\left\|\mathfrak{Q}(\Lambda_w) - \mathfrak{Q}(\tilde\Lambda_{w})\right\|_{F,n}}{|w-z|^2} dw
$
It is proved in  Lemma \ref{lem:proxy_eta} that $\left\|\mathfrak{Q}(\Lambda_w) - \mathfrak{Q}(\tilde\Lambda_{w})\right\|_{F,n}$ tends almost surely to $0$ and that there exists a constant C such that $\left\|\mathfrak{Q}(\Lambda_w) - \mathfrak{Q}(\tilde\Lambda_{w})\right\|_{F,n} \leq C\eta_N$.
Since $\eta_N$ tends to 0,  $\frac{C\eta_N}{|w-z|^2} \leq \frac{\bar C}{|w-z|^2}$ with $\bar C$ a constant independent from $n$. The bound $\frac{\bar C}{|w-z|^2}$ being integrable, we deduce from the dominated convergence theorem that 
\begin{eqnarray}\label{eq:deriv}
\left\|\mathfrak{Q}'(\Lambda_z) - \mathfrak{Q}'(\Lambda_z+i\eta_N I_{N})\right\|_{F,n} \xrightarrow{} 0 \quad \mbox{almost surely}.
\end{eqnarray}

The dominated convergence theorem allows us to prove similarly that
$\left\|id_4 \otimes \Delta [\mathbb{E}\left[\mathfrak{Q}'(\tilde\Lambda_z)\right]-\mathfrak{Q}^{'}(\tilde\Lambda_z)]\right\|_{F,n}$ tends almost surely to $0$. As earlier, one need to bound $\left\|id_4 \otimes \Delta [\mathbb{E}\left[\mathfrak{Q}'(\tilde\Lambda_{w})\right]-\mathfrak{Q}^{'}(\tilde\Lambda_{w})]\right\|_{F,n}$ by a constant.

One can express $\mathfrak{Q}(\tilde\Lambda_z)$ as a block matrix as we did for $\mathfrak{Q}(\Lambda_z)$ allowing us to prove that $\left\|\mathfrak{Q}(\tilde\Lambda_z)\right\|_{F,n}$ is bounded independently from $n$ as we proved that $\left\|\mathfrak{Q}(\Lambda_z)\right\|_{F,n}$ is bounded in the proof of Lemma \ref{lem:proxy_eta}. Thus there exists $\tilde C$, such that 
\begin{eqnarray*}
\left\|\mathbb{E}\left[\mathfrak{Q}(\tilde\Lambda_z)\right]-\mathfrak{Q}(\tilde\Lambda_z)\right\|_{F,n} &\leq& \tilde C.
\end{eqnarray*}
Then, we deduce from \eqref{eq:ineq_concentration} and the dominated convergence theorem, that 
\begin{eqnarray}\label{eq:Q'-esp}
\left\|id_4 \otimes \Delta \left[\mathbb{E}\left[\mathfrak{Q}'(\tilde\Lambda_z)\right]-\mathfrak{Q}^{'}(\tilde\Lambda_z)\right]\right\|_{F,n} \xrightarrow[n\rightarrow +\infty]{} 0 \quad \mbox{almost surely}.
\end{eqnarray}
Once again we prove that $\left\|\mathbb{E}\left[\mathfrak{Q}'(\tilde\Lambda)\right]-\mathfrak{Q}^{\mathrm{\square '}}(\tilde\Lambda)\right\|$ tends almost surely to $0$ with the dominated convergence theorem.
One has from Lemma \ref{lem:equiv-esp}, that, for $\Lambda$ such that $\Im m (\Lambda) \geq \eta_N$,
$$
\left\|\mathbb{E}\left[\mathfrak{Q}(\Lambda)\right]-\mathfrak{Q}^{\mathrm{\square}}(\Lambda)\right\| \leq\frac{1}{n}(1+\gamma_{max}^2\| (\Im m (\Lambda))^{-1} \|^2/\delta ) (2\sqrt{2}\gamma_{max}^3\| (\Im m (\Lambda))^{-1} \|^4 + \kappa \| (\Im m (\Lambda))^{-1} \|^2),
$$
for $0 < \delta < 1$.
For $N$ large enough, $\delta_N = \eta_N^2$ , one has that 
\begin{eqnarray}\label{eq:bound}
\frac{1}{n}(1+\frac{\gamma_{max}^2}{\delta_N}\| (\Im m (\tilde\Lambda_z))^{-1} \|^2 ) (2\sqrt{2}\gamma_{max}^3\| (\Im m (\tilde\Lambda_z))^{-1} \|^4 + \kappa \| (\Im m (\tilde\Lambda_z))^{-1} \|^2) 
\leq \frac{1}{n}(1+\gamma_{max}^2 ) (2\sqrt{2}\gamma_{max}^3 (\eta_N^4 + \kappa  \eta_N^2).
\end{eqnarray}
Since the right term of this inequality tending to $0$ it is bounded independently from $n$ and $\Lambda$, then one has from the dominated convergence theorem that  
\begin{eqnarray}\label{eq:Q'-square}
\left\|\mathbb{E}\left[\mathfrak{Q}'(\tilde\Lambda_z)\right]-\mathfrak{Q}^{\mathrm{\square '}}(\tilde\Lambda_z)\right\| \xrightarrow[n\rightarrow +\infty]{} 0.
\end{eqnarray}
One can combine the triangle inequality with \eqref{eq:Q'-square} and \eqref{eq:Q'-esp}, to obtain that  
\begin{eqnarray}\label{eq:deriv_eq-esp}
\left\|id_4 \otimes \Delta [\mathfrak{Q}^{\square '}(\tilde\Lambda_{z}) - \mathfrak{Q}'(\tilde \Lambda_{z})]\right\|_{F,n} \xrightarrow{} 0 \quad \mbox{almost surely}.
\end{eqnarray}

Since we have bounded $\left\|\mathfrak{Q}'(\Lambda_z) - \mathfrak{Q}'(\Lambda_z+i\eta_N I_{N})\right\|_{F,n}$ and $\left\|id_4 \otimes \Delta [\mathfrak{Q}^{\square '}(\tilde\Lambda_{z}) - \mathfrak{Q}'(\Lambda_{-z})]\right\|_{F,n}$ independently from $n$, one can prove \eqref{eq:deriv} and \eqref{eq:deriv_eq-esp} for $z\in \mathbb{C}\setminus\mathbb{R}^+$ as in the proof of Theorem \ref{thm:equiv_risk}, which completes the proof of Corollary \ref{cor:deriv}.
\end{proof}
Now, to prove the convergence of the various terms appearing in the decomposition of $E_{test}^{\lozenge}(\lambda)$, we need further results beyond the one of Corollary \ref{cor:deriv} that are given below.  
{\CB First, since $\frac{1}{n}\Tr(AA^\star) \leq \| A \|^2$ for every matrix $A$,  it follows from the definition of $id_4 \otimes \Delta$ and Corollary \ref{cor:equiv-ps} that, for $1\leq I,J\leq 4$,}
\begin{equation} \label{eq:conv1}
\left\|\Delta [\mathfrak{Q}_{IJ}(\Lambda)]-\mathfrak{Q}_{IJ}^{\square}(\Lambda)\right\|_{F,n}^2 \xrightarrow{} 0 \quad \mathrm{ a.s.}
\end{equation}
Furthermore, one can deduce from the above equation and Cauchy-Schwarz's inequality that, for any diagonal matrix $D$ with uniformly bounded entries,
\begin{eqnarray}\label{eq:equiv_final}
\frac{1}{n}\Tr[D\mathfrak{Q}_{IJ}(\Lambda)-D\mathfrak{Q}_{IJ}^{\square}(\Lambda)]\xrightarrow{} 0 \quad \mathrm{ a.s.}.
\end{eqnarray}
{\CB
At last, we will also need to use the following corollary.
{\CR
\begin{cor}\label{cor:final}
Let $\Lambda \in \mathrm{D}_N(\mathbb{C})^{+}$ such that $\Im m(\Lambda) > \eta_N$, where $\eta_N$ is defined in \eqref{eq:etaN}. Then, one has that 
$$
\mathbb{E}\left[\frac{1}{n}\left\|D\mathfrak{Q}_{IJ}(\Lambda)\right\|_F^2\right]-\frac{1}{n}\left\|D\mathfrak{Q}^{\square}_{IJ}(\Lambda)\right\|_F^2 \xrightarrow{} 0 \quad \mathrm{ a.s.}
$$
for all $1\leq I,J\leq 4$ and any deterministic diagonal matrix $D$ with uniformly bounded entries.
\end{cor}
}
}
\begin{proof}[Proof of Corollary \ref{cor:final}] {\CB We first remark that}
\begin{eqnarray*}
    \frac{1}{n}\left\|D\mathfrak{Q}^{\square}_{IJ}(\Lambda)\right\|_F^2 - \mathbb{E}[\frac{1}{n}\left\|D\mathfrak{Q}_{IJ}(\Lambda)\right\|_F^2] &\leq& \frac{1}{n}\left\|D\mathfrak{Q}^{\square}_{IJ}(\Lambda)\right\|_F^2 - \frac{1}{n}\left\|\mathbb{E}\left[D\mathfrak{Q}_{IJ}(\Lambda)\right]\right\|_F^2\\ 
    &\leq& \frac{1}{n}\left(\left\|D\mathfrak{Q}^{\square}_{IJ}(\Lambda)\right\|_F - \left\|\mathbb{E}\left[D\mathfrak{Q}_{IJ}(\Lambda)\right]\right\|_F\right)\left(\left\|D\mathfrak{Q}^{\square}_{IJ}(\Lambda)\right\|_F + \left\|\mathbb{E}\left[D\mathfrak{Q}_{IJ}(\Lambda)\right]\right\|_F\right)\\ 
    &\leq& \left(\left\|D\mathfrak{Q}^{\square}_{IJ}(\Lambda)\right\| - \left\|\mathbb{E}\left[D\mathfrak{Q}_{IJ}(\Lambda)\right]\right\|\right)\left(\left\|D\mathfrak{Q}^{\square}_{IJ}(\Lambda)\right\| + \left\|\mathbb{E}\left[D\mathfrak{Q}_{IJ}(\Lambda)\right]\right\|\right)\\ 
    &\leq& \left\| D \right\|\left\|D\mathfrak{Q}^{\square}_{IJ}(\Lambda) - \mathbb{E}\left[D\mathfrak{Q}_{IJ}(\Lambda)\right]\right\|\left(\left\|\mathfrak{Q}^{\square}_{IJ}(\Lambda)\right\| + \left\|\mathbb{E}\left[\mathfrak{Q}_{IJ}(\Lambda)\right]\right\|\right)\\ 
    &\leq& \left\| D \right\|^2\left\|\mathfrak{Q}^{\square}_{IJ}(\Lambda) - \mathbb{E}\left[\mathfrak{Q}_{IJ}(\Lambda)\right]\right\|\left(\left\|\mathfrak{Q}^{\square}_{IJ}(\Lambda)\right\| + \left\|\mathbb{E}\left[\mathfrak{Q}_{IJ}(\Lambda)\right]\right\|\right)
\end{eqnarray*}
One has from Lemma \ref{lem:equiv-esp} and Equation \eqref{eq:bound} that \begin{eqnarray}\label{eq:bound_minus}
\left\|\mathfrak{Q}^{\square}_{IJ}(\Lambda) - \mathbb{E}\left[\mathfrak{Q}_{IJ}(\Lambda)\right]\right\| = O (\eta_N^2).
\end{eqnarray}
Note that the blocks $\mathfrak{Q}(\Lambda)_{IJ}$ are made of products of $Q^\lozenge(-\lambda)$, $H^\lozenge/\sqrt{n}$, $\mathcal X/\sqrt{n}$ and $\mathcal W/\sqrt{p}$.
These matrices can be proved to be bounded in operator norm, indeed $\|H^\lozenge/\sqrt{n}\|$, $\|\mathcal X/\sqrt{n}\|$ and $\|\mathcal W/\sqrt{p}\|$ are bounded {\CR thanks to Assumption \ref{hyp:bounded_profile} and by \cite{BS98}[Theorem 1.1]} as argued in Remark \ref{rem:condsimple}, and $\|Q^\lozenge(-\lambda)\| \leq \lambda^{-1}$ by \cite{hislop2012introduction}[Theorem 5.8]. Hence, we deduce by sub-multiplicativity of the operator norm that the blocks $\|\mathfrak{Q}(\Lambda)_{IJ}\|$ are bounded independently from n. Moreover, one has from Lemmas \ref{lem:Rt} and \ref{lem:res} that 
\begin{eqnarray}\label{eq:bound_plus}
\|\mathfrak{Q}^{\square}_{IJ}(\Lambda)\| \leq \|\mathfrak{Q}^{\square}(\Lambda)\| \leq \eta_N^{-1}.
\end{eqnarray}

Then one can conclude the proof of Corollary \ref{cor:final} by noting that $\eta_N$ tends to $0$ and by combining Equations \eqref{eq:bound_minus} and \eqref{eq:bound_plus}.
\end{proof}

 To obtain the convergence of the various terms appearing in the expression of  $E_{test}^{\lozenge}(\lambda)$, we can now combine the above results.
Thanks to Assumption \ref{hyp:bounded_profile},  the diagonal matrix $D = \mathbb{E}[\frac{1}{\tilde n}\tilde{\mathcal X}_{\tilde{n}}^\top \tilde{\mathcal X}_{\tilde{n}}] = \frac{1}{\tilde n} \deg ( \tilde{\Gamma}_{\tilde{n}}^\top )$  has uniformly bounded entries, and  $\frac{1}{p} \ \mathbb{E}[\frac{1}{\tilde n}\tilde{\mathcal X}_{\tilde{n}} \tilde{\mathcal X}_{\tilde{n}}^\top] = \frac{s^2}{\tilde{n}} I_{\tilde{n}}$ by Assumption \ref{hyp:rowsto}. Hence, combining Equations \eqref{eq:conv1} and \eqref{eq:equiv_final} with Theorem \ref{thm:adaptation} and Corollaries \ref{cor:deriv} and \ref{cor:final} yields  that $\lim_{n \to + \infty} |E_{test}^{\lozenge}(\lambda) - E_{test}^{\square}(\lambda)| = 0$ where
\begin{eqnarray*}
E_{test}^{\square}(\lambda) &=& \sigma^2 + \alpha^2s^2 + \frac{\theta_{chaos}^2(h)\alpha^2}{\theta_{lin}^2(h) p}\Tr[ \mathfrak{Q}^{\square}_{44}] - \frac{\lambda \theta_{chaos}^2(h)\alpha^2}{\theta_{lin}^2(h) p}\Tr[\mathfrak{Q}^{\square '}_{44} ] +\frac{\alpha^2}{p}  \Tr[\frac{1}{\tilde n} \deg ( \tilde{\Gamma}_{\tilde{n}}^\top )(\mathfrak{Q}^{\square}_{34}+I)(\mathfrak{Q}^{\square}_{43} + I) ] \\
                    &+&  \frac{\theta_{chaos}^2(h)\sigma^2}{ n}\Tr[\mathfrak{Q}^{\square '}_{22}  ]+ \frac{\theta_{lin}^2(h)\sigma^2}{ n} \Tr[\frac{1}{\tilde n} \deg ( \tilde{\Gamma}_{\tilde{n}}^\top )\mathfrak{Q}^{\square'}_{33}]-\frac{2\theta_{lin}(h)\alpha^2}{p}  \Tr[\frac{1}{\tilde n} \deg ( \tilde{\Gamma}_{\tilde{n}}^\top ) (\mathfrak{Q}^{\square}_{43}+I)].
\end{eqnarray*}
{\CB This completes the proof of Theorem \ref{thm:main}.}

\section{Numerical experiments}\label{sec:num}

\textcolor{black}{ This section gathers numerical experiments that illustrate the theoretical results of this work. These experiments are conducted for simulated data following a mixture model with  variance profiles as in \eqref{eq:profileMNIST} that are computed from the MNIST database as explained in Section \ref{sec:mixture}, but with images that are subsampled to be of size $p=50$. Apart from Figure~\ref{fig:Laplace}, all the numerical experiments reported in this section are conducted with i.i.d.\ standard Gaussian entries for $X_n'$. We also repeated these experiments with other input distributions satisfying Assumption~\ref{hyp:moments-analytical}; Figure~\ref{fig:Laplace}, for instance, presents the case of the Laplace distribution. In all such cases, we obtained quantitatively very similar results. This observation provides additional empirical support for the robustness of our theoretical findings with respect to the distribution of the design entries, as long as Assumption~\ref{hyp:moments-analytical} is satisfied.}

In order to compute $ E^{\square}_{\text{train}} $ and $ E^{\square}_{\text{test}} $, it is needed to calculate the matrix $ \mathfrak{Q}^\square $. Although we do not possess an explicit analytical expression for this matrix, it can be effectively approximated using a fixed-point algorithm to solve Equation \eqref{eq:fixedpoint}. This algorithm plays a central role in our numerical experiments, as it enables us to obtain accurate approximation of $ E^{\square}_{\text{train}} $ and $ E^{\square}_{\text{test}} $. Moreover, it provides a basis for comparison between different error metrics, namely $ E_{\text{train}} $, $ E^{\square}_{\text{train}} $, and $ E^{\lozenge}_{\text{train}} $ (respectively $ E_{\text{test}} $, $ E^{\square}_{\text{test}} $, and $ E^{\lozenge}_{\text{test}} $). In particular, the construction of $ E^{\square}_{\text{test}} $ follows from a structured modification of the original expression of $ E_{\text{test}} $. Specifically, we replace the sub-blocks of the matrix $ \mathfrak{Q} $ with their corresponding sub-blocks in $ \mathfrak{Q}^\square $. This substitution ensures that $ E^{\square}_{\text{test}} $ incorporates the refined approximation encoded in $ \mathfrak{Q}^\square $, thereby allowing for a direct comparison between the original and approximated errors.
In all the figures of this section, dashed curves correspond to the true predictive error $ E_{\text{test}} $, while solid curves represent the asymptotic equivalent error $ E^{\square}_{\text{test}} $. Remarkably, in all our experiments, the dashed and solid curves overlap almost perfectly, demonstrating the accuracy of this asymptotic equivalent. This near-perfect alignment  further validates the efficacy of the fixed-point algorithm in approximating $ \mathfrak{Q}^\square $ and its impact on predictive error approximation.

\begin{figure}[htbp]
\begin{center}
{\subfigure[]{\includegraphics[width = 0.49\textwidth]{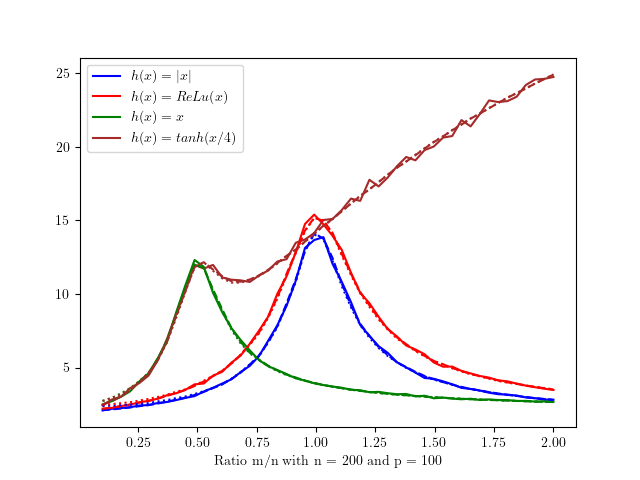}}
}
\hfill
{\subfigure[]{\includegraphics[width =0.49\textwidth]{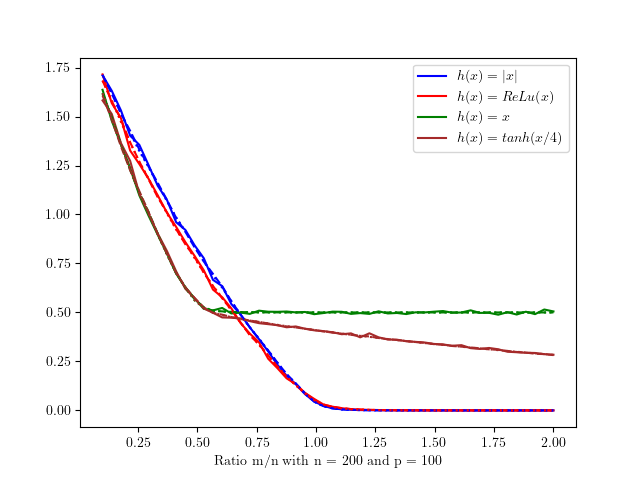}}
}
\end{center}
\caption{\textcolor{black}{(a) Comparison of the predictive risk $E_{test}(\lambda)$ (solid curves) with the lozenge asymptotic equivalent $E_{test}^{\lozenge}(\lambda)$ (dashed curves) the square asymptotic equivalent $E_{test}^{\square}(\lambda)$ (dotted curves) for several activation functions and the ratio $m/n$. (b) Comparison of the training risk $E_{train}(\lambda)$ (solid curves) with the lozenge asymptotic equivalent $E_{train}^{\lozenge}(\lambda)$ (dashed curves) the square asymptotic equivalent $E_{train}^{\square}(\lambda)$ (dotted curves) for several activation functions and the ratio $m/n$.}}\label{fig:activations}
\end{figure}


\begin{figure}[htbp]\label{fig:db_lambdas}
\begin{center}
{\subfigure[]{\includegraphics[width = 0.49\textwidth]{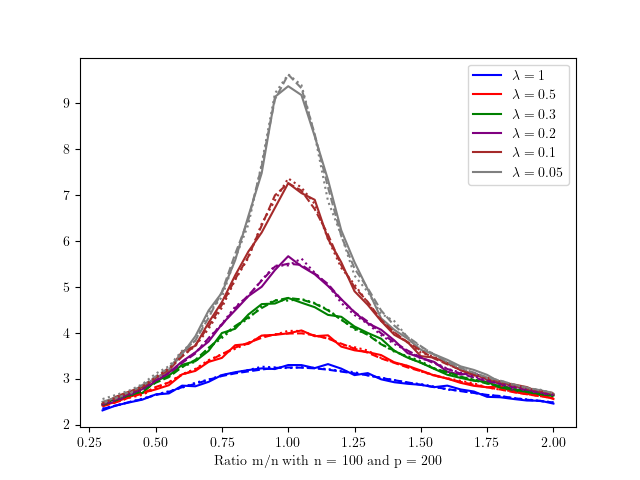}}
}
\hfill
{\subfigure[]{\includegraphics[width =0.49\textwidth]{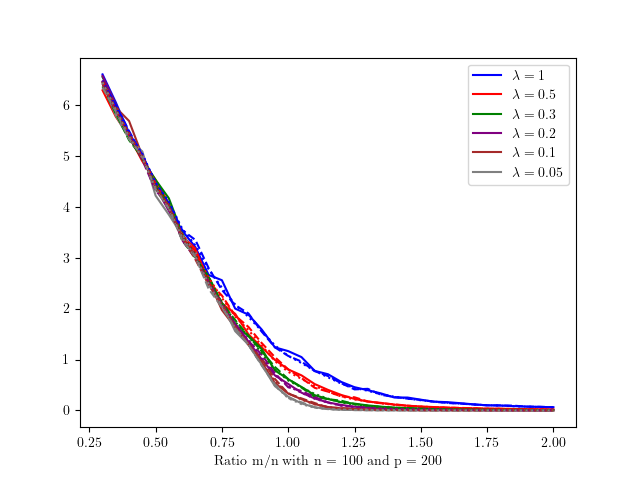}}
}
\end{center}
\caption{\textcolor{black}{(a) Comparison of the predictive risk $E_{test}(\lambda)$ (solid curves) with the lozenge asymptotic equivalent $E_{test}^{\lozenge}(\lambda)$ (dashed curves) the square asymptotic equivalent $E_{test}^{\square}(\lambda)$ (dotted curves)  for different values of $\lambda$ and the ratio $m/n$ with non-gaussian data and $h(x) = x^3$. (b) Comparison of the training risk $E_{train}(\lambda)$ (solid curves) with the lozenge asymptotic equivalent $E_{train}^{\lozenge}(\lambda)$ (dashed curves) the square asymptotic equivalent $E_{train}^{\square}(\lambda)$ (dotted curves) for different values of $\lambda$ and  the ratio $m/n$ with non-gaussian data and $h(x) = x^3$.}} \label{fig:Laplace}
\end{figure}

\begin{figure}[htbp]
\begin{center}
\includegraphics[width=0.8 \textwidth,height=0.5\textwidth]{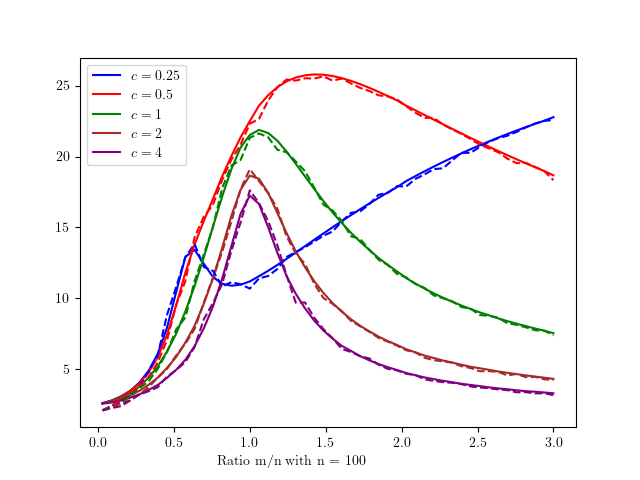}
\caption{Predictive risk $E_{test}(\lambda)$ (dashed curves) and the square asymptotic equivalent $E_{test}^{\square}(\lambda)$ (solid curves) with $\lambda = 0.004$, for  $h(x) = \tanh(c x)$ for different values of $c$ and the ratio $m/n$.}  
\label{fig:tanhs}
\end{center}
\end{figure}

As highlighted in Section \ref{sec:mixture}, the double descent phenomenon still holds for RF regression with a variance profile whenever  $\lambda$ is small enough and for the choice $h(x) = x^3$. This phenomenon is again observed in Figure \ref{fig:activations} that displays the predictive risk of RF ridge regression with $\lambda = 0.004$ for several other activation functions $h$, and the ratio $m/n$ ranging from $0.03$ to $3$ with $n = 100$ and $\tilde{n} = 20$. \textcolor{black}{Our results are established under the assumption that the activation function $h$ is an odd polynomial, which is a limitation compared with the activation functions used in practical neural networks. Nevertheless, we believe that the analysis could be extended to a broader class of functions. To do so, it would be enough to generalize Proposition \ref{prop:H-H_lin} to a wider class of activation functions. We are currently working on such an extension of the results in \cite{DaboMale}, and we are fairly confident that it is feasible. Figure \ref{fig:activations} supports this conjecture: for each of the non-polynomial functions tested, the deterministic equivalents (both lozenge and square) are highly accurate for both the training and the test errors.}

It turns out that there are two locations where the peak in the double descent can appear, namely at $m=n$ and $m=p$ when $m$ varies. This fact has already been  observed in the case of iid data, see e.g.\ \cite{triple-descent}. The authors of \cite{triple-descent} assert that the peak at $m=p$ appears for activation functions $h$ that are close to linear functions, whereas the other peak at $m=n$ appears for highly non-linear functions. This observation seems to still hold in the case of data with a variance profile. Indeed, in Figure \ref{fig:activations}, we observe a peak in the double descent at $p = m$ for $h(x) = x$ and $h(x) = \tanh(x/4)$, whereas the other peak appears for $h(x) = |x|$ and $h(x) = \text{ReLU}(x)$. These results suggest that, for data with a variance profile, the degree of non-linearity of the activation function also plays a fundamental role in determining the location of the double descent peak.

\textcolor{black}{In Figure \ref{fig:Laplace}, we compare the predictive and training risks (respectively $E_{test}(\lambda)$ and $E_{train}(\lambda)$) with their square and lozenge asymptotic equivalents (respectively $E_{test}^{\square}(\lambda)$, $E_{test}^{\lozenge}(\lambda)$ and $E_{train}^{\square}(\lambda)$, $E_{train}^{\lozenge}(\lambda)$), as derived from our theoretical framework, for different values of the ratio $m/n$ ranging from $0.03$ to $3$, with $n = 100$ and $p = 200$. Notably, the curves corresponding to the asymptotic equivalents closely match those of the predictive risk, demonstrating the accuracy of our theoretical results. As shown in Figure \ref{fig:Laplace}(a), the predictive risk curve exhibits a peak whenever $m \approx n$, and this peak becomes sharper as $\lambda$ tends to $0$. This behavior indicates the onset of the double-descent phenomenon for small values of $\lambda$. Moreover, in this experiment, the data were generated exactly as described in Section \ref{sec:mixture} and  the entries of $X'_n$ follow a Laplace distribution  with density
$$
f(x)=\frac{e^{-|x|\sqrt{2}}}{\sqrt{2}}.
$$
The accuracy of the asymptotic equivalents in this setting illustrates the robustness of our results with respect to the distribution of the input data, provided that Assumption \ref{hyp:moments-analytical} is satisfied.}

Figure \ref{fig:tanhs} illustrates the behavior of the predictive risk for  $\lambda = 0.004$, where the  ratio $ m/n $ varies from 0.03 to 3, with fixed values $ n = 100 $ and $ \tilde{n} = 20 $. The considered activation function  is $ h(x) = \tanh(c x) $, where the parameter $ c $ takes values in $\{0.25, 0.5, 1, 2, 4\}$. In this way, as $ c $ decreases, the function $ h(x) $ increasingly behaves as a linear function. This is particularly relevant when analyzing the resulting predictive risk. In agreement with the findings of \cite{triple-descent}, Figure \ref{fig:tanhs} confirms the presence of a double-descent curve whose shape depends on the value of $c$. Specifically, for high values of $ c $, a pronounced peak in the predictive risk is observed around $ m = n $, highlighting the classical interpolation threshold where the number of parameters in the model matches with the number of samples. Conversely, for smaller values of $ c $, the peak shifts towards $ m = p $, reflecting a different regime where the effective model complexity, shaped by the near-linearity of $ h(x) $, influences the behavior of the predictive risk.

In conclusion, our theoretical results and these  numerical results provide, for non-iid data, further insights  into the role of the degree of non-linearity of the activation function  in understanding the generalization error of  neural networks, particularly in over-parameterized regimes.

\appendix

\section*{Appendix}  \label{sec:app}

\renewcommand{\thesection}{\Alph{section}}
\renewcommand{\theequation}{\thesection.\arabic{equation}}
\setcounter{section}{1}
\setcounter{equation}{0}

\subsection{Proof of Lemma \ref{lem:risk_expressions}}

Let us  prove Lemma \ref{lem:risk_expressions} that \rev{allows us to have more tractable expressions} for $E_{train}(\lambda)$ and $E_{test}(\lambda)$

\begin{proof}[Proof of Lemma \ref{lem:risk_expressions}]
{\color{black}
We start from the identity $\hat\theta_\lambda=HQ(-\lambda)Y_n/n$ and from the decomposition $Y_n=X_n\beta_\ast+\varepsilon_n$. Since $Q(-\lambda)=(H^\top H/n+\lambda I_n)^{-1}$, one has
$$
I_n-\frac{1}{n}H^\top HQ(-\lambda)=\lambda Q(-\lambda).
$$
Therefore
\begin{align*}
E_{train}(\lambda)
&=\frac{1}{n}\mathbb E\Big[\|Y_n-H^\top\hat\theta_\lambda\|_2^2\Big]
 =\frac{\lambda^2}{n}\mathbb E\Big[\|Q(-\lambda)Y_n\|_2^2\Big]\\
&=\frac{\lambda^2}{n}\mathbb E\Big[\Tr\big(Q(-\lambda)Y_nY_n^\top Q(-\lambda)\big)\Big].
\end{align*}
Using the independence of $\beta_\ast$, $X_n$, and $\varepsilon_n$, together with
$$
\mathbb E[\beta_\ast\beta_\ast^\top]=\frac{\alpha^2}{p}I_p
\qquad\text{and}\qquad
\mathbb E[\varepsilon_n\varepsilon_n^\top]=\sigma^2 I_n,
$$
we obtain
$$
\mathbb E[Y_nY_n^\top | X_n]=\frac{\alpha^2}{p}X_nX_n^\top+\sigma^2 I_n.
$$
Finally, since $Q'(-\lambda)=Q^2(-\lambda)$, this yields
$$
E_{train}(\lambda)=\frac{\lambda^2\alpha^2}{n}\Tr\Bigg[\mathbb E\Big[Q'(-\lambda)\frac{X_nX_n^\top}{p}\Big]\Bigg]
+\frac{\lambda^2\sigma^2}{n}\Tr\Bigg[\mathbb E\big[Q'(-\lambda)\big]\Bigg].
$$
For the predictive risk, recall that $\tilde Y_{\tilde n}=\tilde X_{\tilde n}\beta_\ast+\tilde\varepsilon$ and that $\hat\theta_\lambda=HQ(-\lambda)Y_n/n$. Hence
\begin{align*}
E_{test}(\lambda)
&=\frac{1}{\tilde n}\mathbb E\Big[\|\tilde Y_{\tilde n}-\tilde H^\top\hat\theta_\lambda\|_2^2\Big]\\
&=\frac{1}{\tilde n}\mathbb E\Big[\big\|\tilde X_{\tilde n}\beta_\ast+\tilde\varepsilon-\frac{1}{n}\tilde H^\top H Q(-\lambda)Y_n\big\|_2^2\Big].
\end{align*}
Expanding the square and using again the independence assumptions removes the mixed noise terms. Moreover,
$$
\frac{1}{\tilde n}\Tr\Bigg[\mathbb E\Big[\frac{1}{p}\tilde X_{\tilde n}\tilde X_{\tilde n}^\top\Big]\Bigg]
=\frac{1}{p\tilde n}\Tr\big[\deg(\tilde\Gamma_{\tilde n})\big].
$$
We therefore obtain
\begin{align*}
E_{test}(\lambda)
&=\sigma^2+\frac{\alpha^2}{p\tilde n}\Tr\big[\deg(\tilde\Gamma_{\tilde n})\big]
 +\frac{\alpha^2}{\tilde n}\Tr\Bigg[\mathbb E\Big[\frac{1}{pn^2}\tilde H^\top H Q(-\lambda)X_nX_n^\top Q(-\lambda)H^\top\tilde H\Big]\Bigg]\\
&\quad+\frac{\sigma^2}{\tilde n}\Tr\Bigg[\mathbb E\Big[\frac{1}{n^2}\tilde H^\top H Q'(-\lambda)H^\top\tilde H\Big]\Bigg]
-\frac{2\alpha^2}{\tilde n n p}\Tr\Bigg[\mathbb E\Big[\tilde X_{\tilde n}X_n^\top Q(-\lambda)H^\top\tilde H\Big]\Bigg],
\end{align*}
which is the announced formula.}
\end{proof}

\subsection{Linearization of a matrix}

{\CB For completeness, we recall below the definition of the linearization of a matrix following the presentation in  \cite{mingo2017free}}.

\begin{definition}\label{def:linearization}
Let $M \in \mathbb{C}\langle X_1, ...,X_N \rangle$ be a polynomial random matrix. We say that $L \in \mathcal{M}_k(\mathbb{C}) \otimes \mathbb{C}\langle X_1, ...,X_N \rangle$ is a linearization of $M$ if it satisfies :
\begin{itemize}
\item[-] $ L = 
 \begin{pmatrix}
0 & A \\
B & C 
\end{pmatrix}
 = a_0\otimes \mathbf{1} + a_1 \otimes X_1 + ... + a_N\otimes X_N$\\
where $a_0,...,a_N \in \mathcal{M}_k(\mathbb{C})$.
\item[-] $A^t,B \in \mathbb{C}^{k-1} \otimes \mathbb{C}\langle X_1, ...,X_N \rangle$.
\item[-] $C \in \mathcal{M}_{k-1}(\mathbb{C}) \otimes \mathbb{C}\langle X_1, ...,X_N \rangle$ is invertible.
\item[-]  $M = AC^{-1}B$.
\end{itemize}
\end{definition}

\subsection{Proof of Theorem \ref{thm:adaptation}}

{\CB The proof of Theorem \ref{thm:adaptation} is  inspired  by results in \cite{bigotmale}. This appendix highlights the main differences between our results and those of \cite{bigotmale} and  the adaptation of their arguments  to the setting of Theorem \ref{thm:adaptation}}. {\CB We first recall in Theorem \ref{thm:bigotmale} below the main result from \cite{bigotmale} on the study of deformed GUE matrices characterized by a variance profile for the model}
$$M_N = \Upsilon \circ A_N + B_N \in \mathbb{R}^{N\times N},$$
 where $A$ represents a GUE matrix,  $\Upsilon$ is a deterministic symmetric matrix and $B$ is a deterministic Hermitian matrix, where {\CB $\Upsilon^{\circ 2} = ({\CB \gamma_N^2(i, j)})$ is a variance profile matrix with real and non-negative entries}. Within this framework, {\CB  a deterministic equivalent of the diagonal of the resolvent of $M_N$ is provided in  \cite{bigotmale} as follows:}


\begin{thm}[\cite{bigotmale} Theorem 1.1 ]\label{thm:bigotmale}
There exists a unique function $G_{M_N}^{\mathrm{\square}}: \mathrm{D}_N(\mathbb{C})^{+} \rightarrow \mathrm{D}_N(\mathbb{C})^{-}$, analytic in each variable, that solves of the following fixed point equation:
$$
G_{M_N}^{\square}(\Lambda)=\Delta\left[\left(\Lambda-\mathcal{R}_N\left(G_{M_N}^{\square}(\Lambda)\right)-Y_N\right)^{-1}\right],
$$
for any $\Lambda \in \mathrm{D}_N(\mathbb{C})^{+}$, with $\mathcal{R}_N(\Lambda)=\underset{i=1, \ldots, N}{\operatorname{diag}}\left(\sum_{j=1}^N \frac{{\CB \gamma_N^2(i, j)}}{N} \Lambda(j, j)\right)$. Let $\gamma_{\max }^2=\max _{i, j} {\CB \gamma_N^2(i, j)}$,  $0<\delta<1$, and consider $\Lambda \in$ $\mathrm{D}_N(\mathbb{C})^{+}$satisfying
$$
\Im m \Lambda \geq \gamma_{\max }\left(\frac{2\sqrt{2}}{N(1-\delta)}\right)^{1 / 5} \mathbb{I}_N.
$$
Then, for any $d>1$, setting
$$
\begin{aligned}
\varepsilon_N(d)= & \sqrt{2} \gamma_{\max } \sqrt{\frac{d \log (N)}{N}}\left\|(\Im m \Lambda)^{-1}\right\|^2 \\
& +\left(1+\frac{\gamma_{\max }^2}{\delta}\left\|(\Im m \Lambda)^{-1}\right\|^2\right) \frac{2 \gamma_{\max }^3\left\|(\Im m \Lambda)^{-1}\right\|^4}{N}
\end{aligned}
$$
we have, for $N \geq 1$,
$$
\mathbb{P}\left(\left\|G_{M_N}(\Lambda)-G_{M_N}^{\square}(\Lambda)\right\| \geq \varepsilon_N(d)\right) \leq 4 N^{1-d}
$$
where $\|\cdot\|$ denotes the operator norm of a matrix, {\CB and $G_{M_N}(\Lambda) = \Delta [(M_N - \Lambda)^{-1}]$.}
\end{thm}

Unfortunately $A_N$ is a GOE matrix rather than a GUE matrix in our case, {\CB thus we cannot directly apply Theorem \ref{thm:bigotmale}}. {\CB In this Appendix, we thus modify some proofs from \cite{bigotmale} to adapt their results to our setting.} 
To this end, let us first prove the existence of a solution to the fixed point equation \eqref{eq:fixedpoint}.
\begin{lem}\label{lem:fixed_point}
There exists a unique deterministic analytic map $\mathfrak{Q}^{\square}: \mathrm{D}_N(\mathbb{C})^{+} \rightarrow \mathrm{D}_N(\mathbb{C})^{+}$ such that 
$$
\mathfrak{Q}^{\square}(\Lambda)=id_4 \otimes \Delta \left[\left( C_N - \Lambda - \mathcal{R}_N(\mathfrak{Q}^{\square}(\Lambda)) \right)^{-1}\right], 
$$ 
for any $\Lambda \in \mathrm{D}_N(\mathbb{C})^{+}$. Moreover, for any $\Lambda, \Lambda^{\prime} \in \mathrm{D}_N(\mathbb{C})^{+}$,
$$
\left\|\mathfrak{Q}^{\square}(\Lambda)-\mathfrak{Q}^{\square}\left(\Lambda^{\prime}\right)\right\| \leq\left\|(\Im m \Lambda)^{-1}\right\|\left\|\left(\Im m \Lambda^{\prime}\right)^{-1}\right\| \times\left\|\Lambda-\Lambda^{\prime}\right\|
$$
\end{lem}
\rev{The proof of Lemma \ref{lem:fixed_point} follows the same lines as \cite{bigotmale}[Lemma 6.3]. More precisely, one studies the map $\psi_\Lambda(G)=id_4\otimes\Delta[(C_N-\Lambda-\mathcal{R}_N(G))^{-1}]$, proves that it is Lipschitz, and shows that it is a contraction whenever $\Im m(\Lambda)>\gamma_{max}$. This yields a unique fixed point in that region; the extension to every $\Lambda\in D_N(\mathbb C)^+$ then follows by analyticity of the resolvent.}


\subsection{Comparison between the expectation of the resolvent and the deterministic equivalent}

{\CB The first step of the proof consists in showing that $\mathfrak{Q}^{\mathrm{\square}}(\Lambda)$ is close to the expectation $\mathbb{E}[\mathfrak{Q}(\Lambda)]$}. To simplify the notation, we write $\gamma_{\max } = \gamma_{\max }^{(L)}$ in this section.

\begin{lem}\label{lem:equiv-esp}
 For all  $\delta \in(0,1)$  and all  {\CB $\Lambda  \in \mathrm{D}_N(\mathbb{C})^{+}$} such that  $\Im m \Lambda \geq \left( \frac{\gamma^2_{max}}{(1-\delta)n} (2\sqrt{2}\gamma_{max}^3 + \kappa) \right)^{\frac{1}{5}} \mathbb{I}_N$  we have \\
$$
\left\|\mathbb{E}\left[\mathfrak{Q}(\Lambda)\right]-\mathfrak{Q}^{\mathrm{\square}}(\Lambda)\right\| \leq\frac{1}{n}(1+\gamma_{max}^2\| (\Im m (\Lambda))^{-1} \|^2/\delta ) (2\sqrt{2}\gamma_{max}^3\| (\Im m (\Lambda))^{-1} \|^4 + \kappa \| (\Im m (\Lambda))^{-1} \|^2)
$$
\end{lem}

The proof of {\CB Lemma \ref{lem:equiv-esp} is then based on the following lemmas that arise from the structure of the proof of Theorem  \ref{thm:bigotmale} in \cite{bigotmale}}.

\begin{lem}[Lemma 5.4 from \cite{bigotmale}]\label{lem:Rt}
For any diagonal matrix $G$ such that $\Im m (G)>0$, one has $\Im m (\mathcal{R}_N(G)) \geq 0$. Moreover, for any diagonal matrix $G$, the following operator norm bound holds
$$
\left\|\mathcal{R}_N(G)\right\| \leq \gamma_{\max }^2\|G\|.
$$
\end{lem}

\begin{lem}[Lemma 5.3 from \cite{bigotmale}]\label{lem:res}
For any Hermitian matrix $A$ and any diagonal matrix $\Lambda$ such that $\Im m \Lambda>$ 0, one has that the matrix $(\Lambda-A)$ is invertible, moreover $\Im m (A-\Lambda)^{-1} > 0$ and
$$
\left\|(A-\Lambda)^{-1}\right\| \leq\left\|(\Im m \Lambda)^{-1}\right\|.
$$
\end{lem}

\begin{lem}[Stein's lemma \cite{haagerup2005new}(Lemma 3.3)]\label{lem:Stein}
Let $f: \mathbb{R}^q \rightarrow \mathbb{C}$ be a continuously differentiable function, and $X_1, \ldots, X_q$ a sequence of independent centered real Gaussian variables with possibly different variances $\mathbb{V}ar\left(X_k\right)=\gamma_k^2$ for $1 \leq k \leq q$. Then, under the conditions that $f$ and its first order derivatives $\frac{\partial}{\partial x_1} f, \ldots, \frac{\partial}{\partial x_q} f$ are polynomially bounded, one has that
$$
\mathbb{E}\left[X_k f\left(X_1, \ldots, X_q\right)\right]=\gamma_k^2 \mathbb{E}\left[\frac{\partial}{\partial x_k} f\left(X_1, \ldots, X_q\right)\right].
$$
\end{lem}

\begin{lem}\label{lem:R-transform}
\rev{For any $\Lambda \in \mathrm{D}_N(\mathbb{C}^{+})$, recalling that $\mathfrak{Q}(\Lambda)=\left(L-\Lambda\right)^{-1}$, we have}
$$
\rev{\mathbb{E}\left[(\Upsilon_L \circ A_N)\mathfrak{Q}(\Lambda)\right]
=
\mathbb{E}\left[\mathcal{R}_N\left(-\mathfrak{Q}(\Lambda)\right)\mathfrak{Q}(\Lambda)\right]-E_N,}
$$
\rev{where}
$$
\rev{\mathcal{R}_N\left(M\right)=\underset{i=1, \ldots, N}{\operatorname{diag}}\left(\sum_{j=1}^N \frac{\left(\gamma^{(L)}_{ij}\right)^2}{N} M(j,j)\right)}
\quad \textcolor{black}{\mbox{and}} \quad
\rev{E_N =\mathbb{E}\left[\left(\frac{1}{N}\big(\Upsilon_L^{\circ 2}+\diag(\Upsilon_L^{\circ 2})\big)\circ \mathfrak{Q}(\Lambda)\right)\mathfrak{Q}(\Lambda)\right].}
$$
\end{lem}

\begin{proof}[Proof of Lemma \ref{lem:R-transform}]
{\color{black}
Write
$$
\Upsilon_L\circ A_N=\sum_{1\le i\le j\le N}\frac{\gamma^{(L)}_{ij}}{\sqrt N}a_{ij}F_{ij},
\qquad
F_{ij}=\frac{E_{ij}+E_{ji}}{1+\delta_{ij}},
$$
where $(a_{ij})_{i\le j}$ are the independent upper-triangular entries of the GOE matrix $A_N$. For a fixed pair $(i,j)$, we apply Stein's lemma to the function
$$
\mathbf a\longmapsto F_{ij}(L(\mathbf a)-\Lambda)^{-1},
$$
with $\mathbf a=(a_{uv})_{u\le v}$. Since
$$
\frac{\partial}{\partial a_{ij}}(L-\Lambda)^{-1}
=-\frac{\gamma^{(L)}_{ij}}{\sqrt N}(L-\Lambda)^{-1}F_{ij}(L-\Lambda)^{-1},
$$
Stein's lemma gives
$$
\mathbb E\big[a_{ij}F_{ij}\mathfrak Q(\Lambda)\big]
=-\frac{\gamma^{(L)}_{ij}}{\sqrt N}\,
\mathbb E\big[F_{ij}\mathfrak Q(\Lambda)F_{ij}\mathfrak Q(\Lambda)\big].
$$
Summing over $1\le i\le j\le 4N$ yields
$$
\mathbb E\left[(\Upsilon_L\circ A_N)\mathfrak Q(\Lambda)\right]
=-\sum_{1\le i\le j\le 4N}\frac{(\gamma^{(L)}_{ij})^2}{N}
\mathbb E\big[F_{ij}\mathfrak Q(\Lambda)F_{ij}\mathfrak Q(\Lambda)\big].
$$
The elementary identity $E_{kl}ME_{uv}=M(l,u)E_{kv}$ shows that the contributions containing $E_{ij}\mathfrak Q(\Lambda)E_{ji}$ and $E_{ji}\mathfrak Q(\Lambda)E_{ij}$ reconstruct the diagonal operator $\mathcal R_N(-\mathfrak Q(\Lambda))\mathfrak Q(\Lambda)$, whereas the terms containing $E_{ij}\mathfrak Q(\Lambda)E_{ij}$ and $E_{ji}\mathfrak Q(\Lambda)E_{ji}$ form the GOE correction $E_N$. This proves the claimed identity.}
\end{proof}

\begin{lem}\label{lem:master_equality}
\rev{(Master equality). For any $\Lambda \in \mathrm{D}_N(\mathbb{C})^{+}$, define}
$$
\rev{\Omega_{L}(\Lambda)=\Lambda+\mathbb{E}\left[\mathcal{R}_N\left(\mathfrak{Q}(\Lambda)\right)\right].}
$$
\rev{Then $\Im m\,\Omega_{L}(\Lambda)\geq \Im m\,\Lambda>0$, hence $C_N-\Omega_L(\Lambda)$ is invertible, and}
$$
\rev{\mathbb{E}[\mathfrak{Q}(\Lambda)]
=
\left(C_N-\Omega_L(\Lambda)\right)^{-1}
+
\left(C_N-\Omega_L(\Lambda)\right)^{-1}O_N,}
$$
\rev{where}
$$
\rev{O_N=\mathbb{E}\left[\mathcal{R}_N(\mathfrak{Q}(\Lambda))\mathfrak{Q}(\Lambda)\right]-\mathbb{E}\left[\mathcal{R}_N(\mathfrak{Q}(\Lambda))\right]\mathbb{E}\left[\mathfrak{Q}(\Lambda)\right]+E_N.}
$$
\end{lem}

\begin{proof}[Proof of Lemma \ref{lem:master_equality}]
{\color{black}
Let $\Omega\in D_N(\mathbb C)$ be such that $C_N-\Omega$ is invertible. Since $L=C_N+\Upsilon_L\circ A_N$, inserting $\pm(C_N-\Omega)$ gives
\begin{eqnarray*}
(C_n - \Omega )^{-1}(\Upsilon_L\circ A_n)(L - \Lambda)^{-1} &=& (C_n - \Omega )^{-1}[  (L - \Lambda) - (C_n - \Omega) + ( \Lambda-\Omega)](L - \Lambda)^{-1} \\
        &=& (C_n - \Omega)^{-1} - (L - \Lambda)^{-1} + (C_n - \Omega)^{-1}( \Lambda-\Omega)(L - \Lambda)^{-1}\\
        &=&(C_N-\Omega)^{-1}-\mathfrak Q(\Lambda)
+(C_N-\Omega)^{-1}(\Lambda-\Omega)\mathfrak Q(\Lambda).
\end{eqnarray*}
Taking expectations and using Lemma \ref{lem:R-transform}, we obtain
\begin{align*}
\mathbb E[\mathfrak Q(\Lambda)]
&=(C_N-\Omega)^{-1}
+(C_N-\Omega)^{-1}\big(\Lambda-\Omega+\mathbb E[\mathcal R_N(\mathfrak Q(\Lambda))]\big)\mathbb E[\mathfrak Q(\Lambda)]\\
&\quad +(C_N-\Omega)^{-1}\Big(\mathbb E[\mathcal R_N(\mathfrak Q(\Lambda))\mathfrak Q(\Lambda)]-\mathbb E[\mathcal R_N(\mathfrak Q(\Lambda))]\mathbb E[\mathfrak Q(\Lambda)]+E_N\Big).
\end{align*}
Choosing now $\Omega=\Omega_L(\Lambda)=\Lambda+\mathbb E[\mathcal R_N(\mathfrak Q(\Lambda))]$ makes the middle term vanish and yields the announced formula. Finally, Lemma \ref{lem:Rt} implies $\Im m\,\mathbb E[\mathcal R_N(\mathfrak Q(\Lambda))]\ge0$, hence $\Im m\,\Omega_L(\Lambda)\ge \Im m\,\Lambda>0$; Lemma \ref{lem:res} therefore guarantees that $C_N-\Omega_L(\Lambda)$ is invertible.}
\end{proof}

\begin{prop}[Gaussian Poincar\'e inequality \cite{haagerup2005new}(Proposition 4.1) ]\label{prop:Poincare}
Let $f: \mathbb{R}^q \rightarrow \mathbb{C}$ be a continuously differentiable function, and $X_1, \ldots, X_q$ a sequence of independent centered real Gaussian variables with possibly different variances $\mathbb{V}$ ar $\left(X_k\right)=\gamma_k^2$ for $1 \leq k \leq q$. Then, under the condition that $f$ and its first order derivatives are polynomially bounded, one has that
$$
\mathbb{V} \operatorname{ar}\left(f\left(X_1, \ldots, X_q\right)\right) \leq \mathbb{E}\left(\left\|\Gamma^{1 / 2} \nabla f\left(X_1, \ldots, X_q\right)\right\|_2^2\right)
$$
where $\Gamma=\operatorname{diag}\left(\gamma_1^2, \ldots, \gamma_q^2\right), \nabla f$ is the gradient of $f$, and $\|\cdot\|_2$ is the standard Euclidean norm of a vector with complex entries.
\end{prop}

\begin{lem}\label{lem:ineq}
(Master inequality). Recall that we denote $\gamma_{\max }=\max _{i, j} \gamma_{i, j}$. For any $\Lambda$ belonging to $\mathrm{D}_N\left(\mathbb{C}^{+}\right)$, we have
$$
\begin{aligned}
\left\|O_N\right\| & \leq \frac{1}{n} (2\sqrt{2}\gamma_{max}^3\left\| (\Im m (\Lambda))^{-1} \right\|^3 + \kappa \left\| (\Im m (\Lambda))^{-1} \right\|^2)
\end{aligned}
$$
\end{lem}

\begin{proof}[Proof of Lemma \ref{lem:ineq}]
{\color{black}
Set
$$
\stackrel{\circ}{Q}_L(\Lambda)=\mathfrak{Q}(\Lambda)-\mathbb E[\mathfrak{Q}(\Lambda)],
\qquad
D_N=\mathcal R_N(\mathfrak Q(\Lambda)),
\qquad
\stackrel{\circ}{D}_N=D_N-\mathbb E[D_N].
$$
By definition of $O_N$,
$$
O_N=\mathbb E[\stackrel{\circ}{D}_N\stackrel{\circ}{Q}_L(\Lambda)]+E_N.
$$
Hence
\begin{equation}\label{eq:on-rev}
\|O_N\|\le \|\mathbb E[\stackrel{\circ}{D}_N\stackrel{\circ}{Q}_L(\Lambda)]\|+\|E_N\|.
\end{equation}
Considering the explicit expression of $E_N$ from Lemma \ref{lem:R-transform}, we obtain the following bound from Assumptions \ref{hyp:bounded_profile} and \ref{hyp:h-bounded}, together with Lemmas \ref{lem:bound Hadamard} and \ref{lem:res}
$$
\|E_N\|\le \frac{\kappa}{n}\| \Im m(\Lambda))^{-1}\|^2.
$$
\textcolor{black}{Indeed, Assumptions \ref{hyp:bounded_profile} and \ref{hyp:h-bounded} and Lemma \ref{lem:bound Hadamard} imply that there exists a deterministic constant $C_{\Upsilon}>0$ such that
$$
\|\Upsilon_L^{\circ 2}\circ B\|\le \kappa\|B\|,
\qquad B\in\mathbb C^{N\times N}.
$$
Hence it is enough to control $\|\mathfrak Q(\Lambda)\|$ in terms of $\|\Im m ( \Lambda)^{-1}\|$ using Lemma \ref{lem:res}.
Now, for every $\Lambda\in \mathrm{D}_N(\mathbb C)^+$, Lemma \ref{lem:Im_pos} gives
$$
\|\mathfrak Q(\Lambda)\|\le \|\Im m(\Lambda)^{-1}\|.
$$
Therefore,
$$
\|\Upsilon_L^{\circ 2}\circ \mathfrak Q(\Lambda)\|
\le \kappa\|\mathfrak Q(\Lambda)\|
\le \kappa\|\Im m( \Lambda)^{-1}\|.
$$
Hence, we get from the explicit expression of $E_N$ obtained in Lemma \ref{lem:R-transform} that 
$$
\|E_N\|\le \frac{\kappa}{n}\| \Im m(\Lambda))^{-1}\|^2.
$$
}
For the first term, the matrix Cauchy-Schwarz inequality gives
$$
\|\mathbb E[\stackrel{\circ}{D}_N\stackrel{\circ}{Q}_L(\Lambda)]\|
\le
\|\mathbb E[\stackrel{\circ}{D}_N\stackrel{\circ}{D}_N^{\ast}]\|^{1/2}
\|\mathbb E[\stackrel{\circ}{Q}_L(\Lambda)\stackrel{\circ}{Q}_L^{\ast}(\Lambda)]\|^{1/2}.
$$
Since $D_N$ is diagonal and $\|\mathfrak Q(\Lambda)\|\le \|(\Im m(\Lambda))^{-1}\|$, this implies
\begin{equation}\label{eq:dq-rev}
\|\mathbb E[\stackrel{\circ}{D}_N\stackrel{\circ}{Q}_L(\Lambda)]\|
\le 2\|(\Im m(\Lambda))^{-1}\|\max_{k\in[N]}\Var(D_N(k,k))^{1/2}.
\end{equation}
To estimate the variance, write the GOE matrix $A_N$ in terms of its independent upper-triangular entries $\mathbf a=(a_{ij})_{i\le j}$ and define, for each $k\in[N]$,
$$
f_k(\mathbf a)=D_N(k,k)=\mathcal R_N\big((L(\mathbf a)-\Lambda)^{-1}\big)(k,k).
$$
Differentiating with respect to $a_{ij}$ yields
$$
\frac{\partial f_k}{\partial a_{ij}}(\mathbf a)
=-\frac{\gamma_{ij}^{(L)}}{\sqrt N}\,\mathcal R_N\big(\mathfrak Q(\Lambda)F_{ij}\mathfrak Q(\Lambda)\big)(k,k),
\qquad
F_{ij}=\frac{E_{ij}+E_{ji}}{1+\delta_{ij}}.
$$
Applying Proposition \ref{prop:Poincare} to $f_k$ gives
\begin{equation}\label{eq:var-rev}
\Var(D_N(k,k))
\le
\frac{\gamma_{\max}^2}{N}
\mathbb E\Bigg[\sum_{i\le j}\Big|\mathcal R_N\big(\mathfrak Q(\Lambda)F_{ij}\mathfrak Q(\Lambda)\big)(k,k)\Big|^2\Bigg].
\end{equation}
The proof now follows the same expansion as in Lemma \ref{lem:R-transform}: one uses the identities
$$
E_{kl}ME_{uv}=M(l,u)E_{kv},
\qquad
\mathcal R_N(M^{\ast})(k,k)=\overline{\mathcal R_N(M)(k,k)},
$$
keeps track of the GOE symmetrization factor $(1+\delta_{ij})^{-1}$, and applies Cauchy-Schwarz to obtain
$$
\sum_{i\le j}\Big|\mathcal R_N\big(\mathfrak Q(\Lambda)F_{ij}\mathfrak Q(\Lambda)\big)(k,k)\Big|^2
\le
\frac{2\gamma_{\max}^4}{N}\|(\Im m(\Lambda))^{-1}\|^4.
$$
Therefore
$$
\Var(D_N(k,k))\le \frac{2\gamma_{\max}^6}{N^2}\|(\Im m(\Lambda))^{-1}\|^4.
$$
Since $N\asymp n$ in the proportional asymptotic regime, this bound is of order $n^{-2}$ and may be written on the $1/n$ scale used throughout the paper. Combining this estimate with \eqref{eq:on-rev} and \eqref{eq:dq-rev} proves that
$$
\|O_N\|\le \frac{1}{n}\Big(2\sqrt{2}\gamma_{\max}^3\|(\Im m(\Lambda))^{-1}\|^3+\kappa\|(\Im m(\Lambda))^{-1}\|^2\Big),
$$
as claimed.}
\end{proof}

\begin{proof}[Proof of Lemma \ref{lem:equiv-esp}]
{\color{black}
By Lemma \ref{lem:master_equality},
$$
\mathbb E[\mathfrak Q(\Lambda)]
=
\left(C_N-\Omega_L(\Lambda)\right)^{-1}+\Theta_N,
\qquad
\Theta_N=\left(C_N-\Omega_L(\Lambda)\right)^{-1}O_N.
$$
Since $\Im m\,\Omega_L(\Lambda)\ge \Im m\,\Lambda>0$, Lemma \ref{lem:res} implies
$$
\left\|\left(C_N-\Omega_L(\Lambda)\right)^{-1}\right\|
\le \| (\Im m(\Omega_L(\Lambda)))^{-1}\|
\le \|(\Im m(\Lambda))^{-1}\|.
$$
Combining this with Lemma \ref{lem:ineq} gives
\begin{equation}\label{eq:theta}
\|\Theta_N\| \leq \frac{1}{n} \Big(2\sqrt{2}\gamma_{max}^3\left\| (\Im m (\Lambda))^{-1} \right\|^4 + \kappa \left\| (\Im m (\Lambda))^{-1} \right\|^3\Big).
\end{equation}
Define
$
\tilde Q_L(\Lambda)=\mathbb E[\mathfrak Q(\Lambda)]-\Theta_N=(C_N-\Omega_L(\Lambda))^{-1} \mbox{ and }\tilde\Lambda=\Lambda+\mathcal R_N(\Theta_N),
$
then
$
\tilde Q_L(\Lambda)=\left(C_N-\tilde\Lambda-\mathcal R_N(\tilde Q_L(\Lambda))\right)^{-1}.
$\\
Equivalently,
$$
\tilde Q_L(\Lambda)=\psi_{\tilde\Lambda}(\tilde Q_L(\Lambda)),
\qquad
\psi_{\tilde\Lambda}(G)=id_4\otimes\Delta\Big[(C_N-\tilde\Lambda-\mathcal R_N(G))^{-1}\Big].
$$
To apply Lemma \ref{lem:fixed_point}, it remains to check that $\tilde\Lambda\in D_N(\mathbb C)^+$. By Lemma \ref{lem:Rt},
$
\|\tilde\Lambda-\Lambda\|=\|\mathcal R_N(\Theta_N)\|\le \gamma_{max}^2\|\Theta_N\|.
$
Hence
\begin{equation}\label{eq:lambda_tilde}
\Im m(\tilde\Lambda)
\ge
\Im m(\Lambda)-\gamma_{max}^2\|\Theta_N\|I_N.
\end{equation}
A sufficient condition ensuring that the right-hand side remains positive is
\begin{equation}\label{eq:bound_Im}
\frac{\gamma_{max}^2}{n}\Big(2\sqrt{2}\gamma_{max}^3\|(\Im m(\Lambda))^{-1}\|^4+\kappa\|(\Im m(\Lambda))^{-1}\|^3\Big)
\le
\frac{1-\delta}{\|(\Im m(\Lambda))^{-1}\|}
\end{equation}
for some $0<\delta<1$. This condition is ensured since it is implied, for $N$ large enough, by the lower bound
\begin{equation}\label{eq:low_bound_Im}
\Im m(\Lambda)\ge \eta_N I_N,
\qquad
\eta_N=\left(\frac{\gamma_{max}^2}{(1-\delta)n}(2\sqrt{2}\gamma_{max}^3+\kappa)\right)^{1/5}.
\end{equation}
Indeed \eqref{eq:low_bound_Im} imply that $\lambda_{min} \geq \eta_N$, where $\lambda_{min}$ is the diagonal entry of $\Lambda$ with the least norm. Moreover, one has that $\lambda_{min} = \frac{1}{\left\| (\Im m (\Lambda))^{-1} \right\|}$, then we have from \eqref{eq:low_bound_Im} that
\begin{eqnarray*} 
\frac{\gamma^2_{max}}{n} (2\sqrt{2}\gamma_{max}^3\left\| (\Im m (\Lambda))^{-1} \right\|^4 + \kappa \left\| (\Im m (\Lambda))^{-1} \right\|^3) &\leq& \frac{\gamma^2_{max}}{n} (2\sqrt{2}\gamma_{max}^3 \eta_N^{-4} + \kappa \eta_N^{-3})
\end{eqnarray*}
Since $\eta_N$ tends to $0$ as $n$ goes to $+\infty$, for $n$ large enough, one has that $\eta_N^{-3} \leq \eta_N^{-4}$, thus 
\begin{eqnarray*} 
\frac{\gamma^2_{max}}{n} (2\sqrt{2}\gamma_{max}^3\left\| (\Im m (\Lambda))^{-1} \right\|^4 + \kappa \left\| (\Im m (\Lambda))^{-1} \right\|^3) &\leq& \frac{\gamma^2_{max}\eta_N^{-4}}{n} (2\sqrt{2}\gamma_{max}^3 + \kappa )\\
&\leq& \eta_N \frac{\gamma^2_{max}\eta_N^{-5}}{n} (2\sqrt{2}\gamma_{max}^3 + \kappa ).
\end{eqnarray*}
Since $\eta^{-5} = \left(\frac{\gamma_{max}^2}{(1-\delta)n}(2\sqrt{2}\gamma_{max}^3+\kappa)\right)^{-1}$, one has that 
\begin{eqnarray*} 
\frac{\gamma^2_{max}}{n} (2\sqrt{2}\gamma_{max}^3\left\| (\Im m (\Lambda))^{-1} \right\|^4 + \kappa \left\| (\Im m (\Lambda))^{-1} \right\|^3) &\leq& (1-\delta) \eta_N\\
&\leq&\frac{1-\delta}{\left\| (\Im m (\Lambda))^{-1} \right\|},
\end{eqnarray*}
hence \eqref{eq:bound_Im} is proved.\\
Under \eqref{eq:bound_Im} and \eqref{eq:lambda_tilde}, we obtain
$$
\Im m(\tilde\Lambda)\ge \delta\,\Im m(\Lambda)\ge \frac{\delta}{\|(\Im m(\Lambda))^{-1}\|}I_N.
$$
Therefore $\tilde\Lambda\in D_N(\mathbb C)^+$ and
$
\|(\Im m(\tilde\Lambda))^{-1}\|\le \frac{1}{\delta}\|(\Im m(\Lambda))^{-1}\|
$, then 
Lemma \ref{lem:fixed_point} yields $\tilde Q_L(\Lambda)=\mathfrak Q^{\square}(\tilde\Lambda)$. Hence
$$
\mathbb E[\mathfrak Q(\Lambda)]=\mathfrak Q^{\square}(\tilde\Lambda)+\Theta_N.
$$
By getting inspiration from the proof of Lemma \ref{lem:Q-analytic}, one can show that the maps $\Lambda\mapsto \mathbb E[\mathfrak Q(\Lambda)]$ and $\Lambda\mapsto \Theta_N$ are analytic on $D_N(\mathbb C)^+$, hence so are $\Lambda\mapsto \tilde\Lambda$ and $\Lambda\mapsto \mathfrak Q^{\square}(\tilde\Lambda)$. Using the Lipschitz bound from Lemma \ref{lem:fixed_point}, we obtain
\begin{align*}
\|\mathbb E[\mathfrak Q(\Lambda)]-\mathfrak Q^{\square}(\Lambda)\|
&\le \|\mathfrak Q^{\square}(\tilde\Lambda)-\mathfrak Q^{\square}(\Lambda)\|+\|\Theta_N\|\\
&\le \|(\Im m(\Lambda))^{-1}\|\,\|(\Im m(\tilde\Lambda))^{-1}\|\,\|\tilde\Lambda-\Lambda\|+\|\Theta_N\|\\
&\le \Big(1+\gamma_{max}^2\|(\Im m(\Lambda))^{-1}\|^2/\delta\Big)\|\Theta_N\|.
\end{align*}
Inserting \eqref{eq:theta} gives the desired estimate.}
\end{proof}

\subsection{Concentration of the resolvent}
The following step consists in proving that the block of the resolvent concentrate around their diagonal.
\begin{lem}[Lemma 7.3 from \cite{bigotmale}]
Let $\Lambda \in \mathrm{D}_N(\mathbb{C})^{+}$. Then, for any pair of unit vectors $v, w$ (that is $\|v\|_2=$ $\left.\|w\|_2=1\right)$, one has that, for all $t>0$
$$
\mathbb{P}\left(\left|v^*\left(\left(L-\Lambda\right)^{-1}-\mathbb{E}\left[\left(L-\Lambda\right)^{-1}\right]\right) w\right| \geq t\right) \leq 4 \exp \left(-N \frac{t^2\left\|(\Im m \Lambda)^{-1}\right\|^{-4}}{2 \gamma_{\max }^2}\right)
$$
where $\gamma_{\max }^2$ is the maximum of the variances in the profile $\Gamma_N$. 
\end{lem}
Note that $id_4 \otimes \Delta [\mathfrak{Q}(\Lambda)-\mathbb{E}\left[\mathfrak{Q}(\Lambda)\right]]$ is the following matrix with diagonal \rev{blocks}
$$id_4 \otimes \Delta[\mathfrak{Q}(\Lambda)-\mathbb{E}\left[\mathfrak{Q}(\Lambda)\right]] = 
\begin{pmatrix} \Delta [\mathfrak{Q}_{11}-\mathbb{E}\left[\mathfrak{Q}_{11}\right]] & 0 & 0 & 0\\
0 &  \Delta [\mathfrak{Q}_{22}-\mathbb{E}\left[\mathfrak{Q}_{22}\right]]  & 0 & 0 \\
0 & 0  & \Delta [\mathfrak{Q}_{33}-\mathbb{E}\left[\mathfrak{Q}_{33}\right]] & \Delta [\mathfrak{Q}_{34}-\mathbb{E}\left[\mathfrak{Q}_{34}\right]] \\
0 & 0 & \Delta [\mathfrak{Q}_{43}-\mathbb{E}\left[\mathfrak{Q}_{43}\right]]  & \Delta [\mathfrak{Q}_{44}-\mathbb{E}\left[\mathfrak{Q}_{44}\right]]
\end{pmatrix},
$$
where $Q_{IJ}$ denotes the $_{IJ}$ sub-block of $\mathfrak{Q}(\Lambda)$.
Let us denote the $IJ$ block of $id_4 \otimes \Delta[\mathfrak{Q}(\Lambda)-\mathbb{E}\left[\mathfrak{Q}(\Lambda)\right]]$ by $M_{IJ}$, then 
$$
\| id_4 \otimes \Delta[\mathfrak{Q}(\Lambda)-\mathbb{E}\left[\mathfrak{Q}(\Lambda)\right]] \| \leq \| M_{11} \| +  \| M_{22} \| +  \| M_{33} \| +  \| M_{44} \| +  \| M_{34} \| +  \| M_{43} \|.
$$
The sub-blocks $M_{IJ}$ being diagonal matrices, one has that $\| M_{IJ}\| = \max\limits_{k}|M_{IJ}(k,k)|$, thus 
$$
\| id_4 \otimes \Delta[\mathfrak{Q}(\Lambda)-\mathbb{E}\left[\mathfrak{Q}(\Lambda)\right]] \| \leq 6 \max _{I,J,k}|M_{IJ}(k,k)|.
$$
This maximum is taken over $N+2p$ elements, thus combining Inequality (7.3) with a union bound yields the following concentration inequality: for all $t>0$,
\begin{eqnarray}\label{eq:ineq_concentration}
\mathbb{P}\left(\left\|id_4 \otimes \Delta [\mathfrak{Q}(\Lambda)-\mathbb{E}\left[\mathfrak{Q}(\Lambda)\right]]\right\| \geq t\right) \leq 4 (N+2p) \exp \left(-N \frac{t^2\left\|(\Im m \Lambda)^{-1}\right\|^{-4}}{72 \gamma_{\max }^2}\right)
\end{eqnarray}
Hence, taking 
\begin{eqnarray*}
t_N^{(1)} &=& \frac{1}{n}(1+\gamma_{max}^2\| (\Im m (\Lambda))^{-1} \|^2/\delta ) (2\sqrt{2}\gamma_{max}^3\| (\Im m (\Lambda))^{-1} \|^4 + \kappa \| (\Im m (\Lambda))^{-1} \|^2)\\
t_N^{(2)}&=&6\gamma_{max}\| \Im m (\Lambda)^{-1}\|^2\sqrt{\frac{2}{N}\log\left( \frac{N+2p}{N^{1-d}}\right)},
\end{eqnarray*}
(for some $\left.d>1\right)$, one finally obtains from \eqref{eq:ineq_concentration} and Lemma \ref{lem:equiv-esp} that
$$
\mathbb{P}\left(\left\|id_4 \otimes \Delta\left[\mathfrak{Q}^\square(\Lambda) - \mathfrak{Q}(\Lambda)\right]\right\| \geq t_N^{(2)}+t_N^{(1)}\right) \leq 4 N^{1-d},
$$
which finishes the proof of Theorem \ref{thm:adaptation}.

Note that $\Lambda_\lambda$ is a real matrix then $\Im m(\Lambda_\lambda) = 0$, which means that it does not meet with the assumptions of Theorem \ref{thm:adaptation}. Therefore one cannot obtain a deterministic equivalent of $\mathfrak{Q}(\Lambda_\lambda)$ from this theorem. However one can bypass this issue by going thanks to $\mathfrak{Q}(\Lambda_\lambda+i\eta_n I_{N})$, with $\eta_n = \left( \frac{\gamma^2_{max}}{(1-\delta)n} (2\sqrt{2}\gamma_{max}^3 + \kappa) \right)^{\frac{1}{5}}$. As a matter of fact, $\mathfrak{Q}(\Lambda_\lambda+i\eta_n I_{N})$ satisfies the assumption of Theorem \ref{thm:adaptation}, thus one has a deterministic equivalent for this matrix, which is $\mathfrak{Q}^\square(\Lambda_\lambda+i\eta_n I_{N})$. Moreover we prove in the following lemma that $\mathfrak{Q}(\Lambda_\lambda+i\eta_n I_{N})$ is close from  $\mathfrak{Q}(\Lambda_\lambda)$ when $n$ goes to infinity so that $\mathfrak{Q}^\square(\Lambda_\lambda+i\eta_n I_{N})$ is a good approximation of $\mathfrak{Q}(\Lambda_\lambda)$ in high dimension.

\begin{lem}\label{lem:proxy_eta}
For $\lambda > 0$ the following limits hold almost surely
\begin{eqnarray*}
\left\|\mathfrak{Q}(\Lambda_\lambda) - \mathfrak{Q}(\Lambda_\lambda+i\eta_n I_{N})\right\|_{F,n} \xrightarrow{} 0 \quad \mbox{ and } \quad \left\|id_4 \otimes \Delta [\mathfrak{Q}^\square(\Lambda_\lambda) - \mathfrak{Q}(\Lambda_\lambda+i\eta_n I_{N})]\right\|_{F,n} \xrightarrow{} 0.
\end{eqnarray*}

\end{lem}

\begin{proof}[Proof of Lemma \ref{lem:proxy_eta}]
We first remark that
\begin{eqnarray*} 
\mathfrak{Q}(\Lambda+i\eta_n I_{N}) &=& \mathfrak{Q}(\Lambda) - i\eta_n \mathfrak{Q}^2(\Lambda) + \eta^2_n \mathfrak{Q}(\Lambda)\mathfrak{Q}(\Lambda+i\eta_n I_{N})\mathfrak{Q}(\Lambda).
\end{eqnarray*} 
Thus one has from the triangle inequality and Lemma \ref{lem:fro_sub} that 
\begin{eqnarray*} 
\left\|\mathfrak{Q}(\Lambda+i\eta_n I_{N}) - \mathfrak{Q}(\Lambda)\right\|_{F,n} &\leq&  \eta_n \left\|\mathfrak{Q}^2(\Lambda)\right\|_{F,n} + \eta^2_n \left\|\mathfrak{Q}(\Lambda)\mathfrak{Q}(\Lambda+i\eta_n I_{...})\mathfrak{Q}(\Lambda)\right\|_{F,n} \\
            &\leq&  \eta_n \left\|\mathfrak{Q}(\Lambda)\right\|_{F,n}^2 + \eta^2_n \left\|\mathfrak{Q}(\Lambda)\right\|_{F,n}\left\|\mathfrak{Q}(\Lambda)\right\| \left\|\mathfrak{Q}(\Lambda+i\eta_n I_{N})\right\|.
\end{eqnarray*} 
Denote by $\mathfrak{Q}(\Lambda)^{ij}$ the $ij$ block of $\mathfrak{Q}(\Lambda)$, then one has that $\left\|\mathfrak{Q}(\Lambda)\right\|_{F,n} \leq \left\|\mathfrak{Q}(\Lambda)\right\| \leq \sum_{ij} \left\|\mathfrak{Q}(\Lambda)^{ij}\right\|$. Note that the blocks $\mathfrak{Q}(\Lambda)^{ij}$ are made of products of $Q^\lozenge(-\lambda)$, $H^\lozenge/\sqrt{n}$, $\mathcal X/\sqrt{n}$ and $\mathcal W/\sqrt{p}$.
These matrices can be proved to be bounded in operator norm, indeed $\|H^\lozenge/\sqrt{n}\|$, $\|\mathcal X/\sqrt{n}\|$ and $\|\mathcal W/\sqrt{p}\|$ are bounded {\CR thanks to Assumption \ref{hyp:bounded_profile} and by \cite{BS98}[Theorem 1.1]} as argued in Remark \ref{rem:condsimple},  
and $\|Q^\lozenge(-\lambda)\| \leq \lambda^{-1}$ by \cite{hislop2012introduction}[Theorem 5.8]. Hence, we deduce by sub-multiplicativity of the operator norm that the blocks $\|\mathfrak{Q}(\Lambda)^{ij}\|$ are bounded, which means that there exists a constant $c$, such that $\| \mathfrak{Q}(\Lambda_\lambda)\|_{F,n} \leq c$. Moreover, one can deduce from Lemma \ref{lem:res} that $\left\|\mathfrak{Q}(\Lambda+i\eta_n I_{N})\right\|\leq \eta_n^{-1}$.
Thus, there is a constant $C$, such that 
\begin{eqnarray*} 
\left\|\mathfrak{Q}(\Lambda_\lambda+i\eta_n I_{N}) - \mathfrak{Q}(\Lambda_\lambda)\right\|_{F,n} &\leq& C\eta_n.
\end{eqnarray*} 
Since $\eta_n$ tends to $0$ as $n$ goes to $+\infty$, one has that 
$$
\left\|\mathfrak{Q}(\Lambda_\lambda) - \mathfrak{Q}(\Lambda_\lambda+i\eta_n I_{N})\right\|_{F,n} \xrightarrow{} 0 \quad \mbox{ almost surely}.
$$
Finally, combining this above equation with the triangle inequality and Corollary \ref{cor:equiv-ps}, gives us that 
$$
\left\|id_4 \otimes \Delta [\mathfrak{Q}^\square(\Lambda_\lambda) - \mathfrak{Q}(\Lambda_\lambda+i\eta_n I_{N})]\right\|_{F,n} \xrightarrow{} 0 \quad \mbox{ almost surely}.
$$
This completes the proof of Lemma \ref{lem:proxy_eta}.
\end{proof}
{\CB We have finally given all the main ingredients needed to prove Theorem \ref{thm:adaptation}, which concludes this Appendix}.

\bibliographystyle{plain}
\bibliography{biblio.bib}


\end{document}